\documentclass[twoside]{article}

\usepackage[accepted]{arxiv}
\usepackage[round]{natbib}

\usepackage[hidelinks]{hyperref}

\usepackage{amsmath,amssymb}
\usepackage{amsthm}
\newtheorem{lemma}{Lemma}
\newtheorem{theorem}{Theorem}
\newtheorem{proposition}{Proposition}
\newtheorem{definition}{Definition}
\newtheorem{assumption}{Assumption}
\newtheorem{example}{Example}

\usepackage{algorithm}
\usepackage{algpseudocode}
\usepackage{bbm}

\newcommand{\pcr}[1]{{\normalfont\text{\fontfamily{pcr}\selectfont #1}}}

\usepackage{xcolor}

\newcommand{\Poly}{\mathrm{Poly}}

\usepackage{xr}
\makeatletter

\newcommand*{\addFileDependency}[1]{
\typeout{(#1)}
\@addtofilelist{#1}
\IfFileExists{#1}{}{\typeout{No file #1.}}
}\makeatother

\renewenvironment{abstract}
  {\small
   \begin{center}
    \bfseries \abstractname
   \end{center}
   \list{}{%
     \setlength{\leftmargin}{10mm}
     \setlength{\rightmargin}{10mm}
    }
   \item\relax}
{\endlist}
\usepackage{authblk}

\usepackage[paper=letterpaper,
            marginparwidth=1in,     
            marginparsep=0in,       
            margin=1in,               
            bottom=1in  
        ]{geometry}
\begin{document}

%
\runningtitle{Smoothness-Adaptive Dynamic Pricing}

%
\runningauthor{}

\onecolumn
\title{Smoothness-Adaptive Dynamic Pricing with Nonparametric Demand Learning}
\date{}

\author[$\diamond$]{Zeqi Ye}
\author[$\dagger$]{Hansheng Jiang}
\affil[$\diamond$]{Nankai University}
\affil[$\diamond$]{\texttt{zye@mail.nankai.edu.cn}}

\affil[$\dagger$]{University of Toronto} 
\affil[$\dagger$]{\texttt{hansheng.jiang@utoronto.ca}}
\maketitle

\begin{abstract}
We study the dynamic pricing problem where the demand function is nonparametric and H\"older smooth, and we focus on adaptivity to the unknown H\"older smoothness parameter $\beta$ of the demand function. Traditionally the optimal dynamic pricing algorithm heavily relies on the knowledge of $\beta$ to achieve a minimax optimal regret of $\widetilde{O}(T^{\frac{\beta+1}{2\beta+1}})$. However, we highlight the challenge of adaptivity in this dynamic pricing problem by proving that no pricing policy can adaptively achieve this minimax optimal regret without knowledge of $\beta$. Motivated by the impossibility result, we propose a self-similarity condition to enable adaptivity. Importantly, we show that the self-similarity condition does not compromise the problem's inherent complexity since it preserves the regret lower bound $\Omega(T^{\frac{\beta+1}{2\beta+1}})$. Furthermore, we develop a smoothness-adaptive dynamic pricing algorithm and theoretically prove that the algorithm achieves this minimax optimal regret bound without the prior knowledge $\beta$.
\end{abstract}

\section{INTRODUCTION}

Dynamic pricing, the practice of adjusting prices in real-time based on varying market demand, has become an integral strategy in domains like e-commerce and transportation. An effective dynamic pricing model needs to adequately balance the exploration by learning demand at various prices and the exploitation by optimizing prices based on observed price and demand data. We consider a canonical dynamic pricing problem with nonparametric demand learning. At the period $t$, the decision maker chooses a price $p_t$ and observes a noisy demand $d_t$, where $\mathbb{E}[d_t |p_t = p] = f(p)$ for some unknown function $f:\mathcal{P} \rightarrow \mathbb{R}_{\geq 0}$ mapping from price set $\mathcal{P}$ to demand. The goal of dynamic pricing is to maximize the total revenue collected over a finite time horizon. The performance of a dynamic pricing policy or algorithm is measured by the cumulative regret when compared with the maximal revenue in hindsight. More broadly framed as an online optimization problem, the dynamic pricing problem features nonparametric demand learning in that $f$ can be of any functional form and continuous action space where price can be chosen at any value in a given price interval. Dynamic pricing problem has been an active topic for decades \citep{kleinberg2003value} and has found numerous applications in retailing, auctions, and advertising \citep{den2015dynamic}.

Without much regularity assumption on the demand function $f$, the optimal regret is shown to be $\widetilde{O}(T^{2/3})$. This regret rate can be improved to $\widetilde{O}(T^{1/2})$ if the uniqueness of the maximum and certain local concavity property of the revenue function $r(p) = p\cdot f(p)$ is imposed. However, such uniqueness assumption can be restrictive in practice and therefore other regularity assumptions, notably the smoothness condition, of the demand functions are considered. Nonetheless, a prevalent limitation in these methodologies is the presupposed exact knowledge of the H\"older smoothness level $\beta$. In reality, such assumptions are frequently misaligned with the complexities of real-world applications, thus constraining the practical applicability of these algorithms. Against this backdrop, our work distinguishes itself by delving into the uncharted territories of adaptability in dynamic pricing. Specifically, we address the pressing challenge of how to adapt when the H\"older smoothness level $\beta$ is not known.


Facing the challenge of unknown smoothness parameter, it is natural to ask the following question:

\emph{Can we design a dynamic pricing strategy that does not require the prior knowledge of $\beta$ while maintaining the optimal regret of $\widetilde{O}(T^{\frac{\beta+1}{2\beta+1}})$?}

Our answer to this question is two-fold: on the one hand, it is impossible to achieve adaptivity without imposing additional assumptions; on the other hand, we identify a novel condition that achieves adaptivity without reducing the original pricing problem's complexity. Our contributions in this paper can be summarized as follows:
\begin{itemize}
\item {\bf Characterizing Adaptivity Challenge:} We formally characterize the challenge of adaptivity. In particular, we prove that without additional conditions, achieving the optimal regret for functions without knowing the H\"older smoothness parameter is impossible. We show that one algorithm with optimal regret for a certain H\"older smoothness parameter can have sub-optimal regret when directly applied to function class with lower H\"older smoothness levels.

\item {\bf Proposing a Self-Similarity Condition:} To make adaptivity possible, we propose a self-similarity condition, which serves as a dual to the H\"older smoothness assumption. Furthermore, our analysis reveals notable properties of the self-similarity condition, in particular regarding its practical applicability and sustenance of the dynamic pricing problem’s complexity. We find that the self-similarity condition not only enables adaptivity but also does not decrease the intrinsic complexity of the original pricing problem in that the lower bound $\Omega(T^{\frac{\beta+1}{2\beta+1}})$ does not change.

\item {\bf Optimal Minimax Regret Rate:}  We design a Smoothness-Adaptive Dynamic Pricing (\pcr{SADP}) algorithm by incorporating a dedicated phase for the estimation of the smoothness parameter. Under the self-similarity condition, we establish a tight confidence interval for the estimated H\"older smoothness parameter. We derive an optimal regret bound $\widetilde{O}(T^{\frac{\beta+1}{2\beta+1}})$ that matches the same optimal bound obtained by previous algorithms that require the knowledge of $\beta$.

\end{itemize}

\paragraph{Organization and Notation} In Section \ref{section:literature}, we introduce related literature on dynamic pricing, bandits, and statistics. In Section \ref{section:problem}, we explicitly formulate the dynamic pricing problem under H\"older smooth demand functions and introduce the adaptivity problem by first presenting the non-adaptive dynamic pricing algorithm. We discuss in-depth the adaptivity challenge in Section \ref{section:adaptivity} and present two key favorable properties of the self-similarity condition. In Section \ref{section:algorithm}, we present our smoothness adaptive dynamic pricing algorithm and give a detailed regret analysis. Lastly, we conclude the paper with discussions and future directions in Section \ref{section:conclusion}. 

Throughout the paper, the vectors are column vectors unless specified otherwise. The notation $\Vert x \Vert$ denotes the $L_2$ norm of vector $x$, and given matrix $A$, the notation $\Vert x\Vert_A = (x^T Ax)^{1/2}$ denotes the $A$-norm of vector $x$. For matrix $A$, $\Vert A\Vert = \sup_{x\neq 0}\Vert x^TAx\Vert /\Vert x\Vert $ denotes the $L_2$ operator norm of matrix $A$. We employ the notation $O(\cdot)$, $\Omega(\cdot)$, $\Theta(\cdot)$ to conceal constant factors, and $\widetilde{O}(\cdot)$, $\widetilde{\Omega}(\cdot)$, $\widetilde{\Theta}(\cdot)$ are used to mask both constant and logarithmic factors.

\section{Related Literature}
\label{section:literature}
\paragraph{Dynamic Pricing with Demand Learning}
Motivated by the applications in e-commerce and transportation, numerous works have studied dynamic pricing with continuous price space and demand learning \citep{kleinberg2003value, besbes2009dynamic,broder2012dynamic,besbes2012blind,keskin2014dynamic,chen2022primal}. The crux of non-contextual dynamic pricing lies in modeling and learning the unknown price and demand relationship. Earlier works mainly focus on parametric demand models with additional concavity property of the revenue function where a regret $\widetilde{O}(\sqrt{T})$ is typically shown to be optimal. For nonparametric demand models,  $\widetilde{O}(T^{\frac{k+1}{2k+1}})$ regret can be achieved if the demand function is $k$ times differentiable reward function for some integer $k>0$, and moreover a matching lower bound of $\Theta(T^{\frac{k+1}{2k+1}})$ can be established \citep{wang2021multimodal}. However, the smoothness level $k$ needs to be known prior to the algorithmic design, and it is thus unclear if existing algorithms are able to adapt to different smoothness levels. Our work improves upon \citet{wang2021multimodal} by proposing a smoothness-adaptive dynamic pricing algorithm with the same minimax optimal regret rate and additionally, we extend the integer $k$ to more generally $\beta$-smooth for any $\beta \in \mathbb{R}_+$.

In certain applications, consumer or product features, also known as contexts, are available and can be parametrized into the demand valuation \citep{qiang2016dynamic, javanmard2017perishability,cohen2020feature, ban2021personalized, xu2021logarithmic}. The landscape of regret analysis in contextual cases typically ranges from $\log(T)$ to $\widetilde{O}(\sqrt{T})$ depending on different parametric or semiparametric assumptions on demand valuation and market noise. The smoothness level of both the demand function and the noise function may affect the regret bound, and theoretical results for adaptively learning the smoothness level are not known \citep{fan2022policy, bu2022context}.

\paragraph{Continuum-Armed Bandit Problems}
Dynamic pricing is closely related to the continuum-armed bandit problem, where the actions are not discrete but rather lie in a continuous space as in the case of the continuous price space. Adaption to H\"older smoothness level $\beta$ while achieving the minimax regret rate has been considered in continuum-armed bandits as well. It is shown in \cite{locatelli2018adaptivity} and \cite{hadiji2019polynomial} that adaptivity for free is generally impossible. Our non-adaptivity result for dynamic pricing shares the same spirit as in the continuum-armed bandit problem but requires different construction of function classes in the arguments. \cite{liu2021smooth} propose to use a general model approach for bandit problems, but the analysis only applies to the subcase of $\beta\leq 1$. Due to non-adaptivity, additional assumptions are therefore necessary for establishing adaptivity. Specifically, the assumption of self-similarity emerges as a promising candidate because it has been demonstrated to maintain the minimax regret rates in both continuum-armed bandits \citep{cai2022stochastic} and contextual bandits \citep{gur2022smoothness} scenarios.

\paragraph{Adaptivity in Statistics}
More broadly, adaptive inference and adaptive estimators have been widely considered in statistics, but less is known if these techniques are suited for regret minimization.
While several structural conditions have profound implications in nonparametric regression, such as monotonicity, concavity, as discussed in \cite{cai2013adaptive}, introducing any of these assumptions may either significantly diminish the problem's complexity or do not directly contribute to the learning of the smoothness parameter \citep{slivkins2019introduction,cai2022stochastic}. Consequently, with any of these structural assumptions at play, the minimax regret operates at the parametric rate, making it agnostic to smoothness variations.

\section{PRELIMINARIES}
\label{section:problem}
\paragraph{Problem Description} We consider the dynamic pricing problem with demand learning over a finite time horizon of length $T$. At every time period $t = 1,\dots, T$, the seller selects a price $p_t \in [p_{\min}, 1]$, where $0<p_{\min}<1$ is a predetermined price lower bound and the price upper bound is normalized to $1$ without loss of generality. After the seller sets the price, the customers then arrive and a randomized demand $d_t \in [0, d_{\max}]$ is incurred. The randomized demand $d_t$ given price is determined by a demand function $f:[p_{\min}, 1] \rightarrow [0, d_{\max}]$ and some random market noise, and the expectation of the randomized demand $\mathbb{E}[d_t|p_t=p] = f(p)$. The noise in demand $d_t - f(p)$ follows a sub-gaussian distribution with respect to some parameters. The revenue collected at time $t$ is $r_t = p_t \cdot d_t$, and the expected revenue given $p_t$ is $p_t \times f(p_t)$. 

As is common in previous literature on pricing \citep{wang2021multimodal,bu2022context}, the H\"older smoothness assumption is used to constrain the volatility of the demand function $f$ in any given region. Throughout the paper, the demand function $f$ is assumed to belong to the H\"older smooth function class $\mathcal{H}(\beta,L)$ for certain $\beta,L>0$ that are defined as follows.

\begin{definition}[H\"older Smooth Function Class]
   The H\"older class of functions $\mathcal{H}_0\left(\beta,L\right)$ is defined to be the set of $w\left(\beta\right)$ times continuously differentiable functions $g : \left[p_{\min},1\right] \rightarrow \mathbb{R}$ such that for any $p,p' \in \left[p_{\min},1\right]$,
   $$
   \sup_{p\in \left[p_{\min},1\right]} \left|g^{(k)}\left(p\right) \right| \leq L, \forall 0\leq k \leq w\left(\beta\right),
   $$
   $$
   \left|g^{w\left(\beta\right)}\left(p\right)-g^{w\left(\beta\right)}\left(p'\right) \right| \leq L \cdot \left|x-x' \right|^{\beta - w\left(\beta\right)},
   $$
   where $w(\beta)$ is the largest integer that is strictly smaller than $\beta$.
   We further define the function class $\mathcal{H}\left(\beta,L\right)$ as
   $$
   \mathcal{H}\left(\beta,L\right) = 
   \begin{cases} 
   \mathcal{H}_0\left(\beta,L\right), & \text{if }0<\beta < 1,\\
   \mathcal{H}_0\left(\beta,L\right) \cap \mathcal{H}_0\left(1,L\right),& \text{if } \beta \geq 1. \\
    \end{cases}
   $$
\end{definition}

\paragraph{Policy and Regret}
An admissible dynamic pricing policy $\pi$ over T selling periods is a sequence of $T$ random functions $\pi_1, \pi_2, \cdots, \pi_T$ such that $\pi_t:(p_1,d_1,\cdots,p_{t-1},d_{t-1})\mapsto p_t$ is a mapping function that maps the history prior to time t to a price $p_t$. Since the demand function belongs to $\mathcal{H}(\beta, L)$ and thus continuous over $[p_{\min},1]$, there exists some optimal price $p^* \in \mathop{\arg\max}_{p \in [p_{\min},1]} \mathbb{E}[r_t|p_t=p]$. Note that here we do not require the optimal price to be unique.

The performance of dynamic pricing policies is evaluated by the cumulative regret defined as follows. For an admissible dynamic pricing policy $\pi$ over T selling periods, the regret $R^{\pi}(T)$ over a time horizon $T$ is 
    $$
    R^{\pi}(T) = \mathbb{E}^{\pi}\left[\sum_{t=1}^T\{p^*f(p^*) - p_tf(p_t)\}\right],
    $$
    where the price sequence $\{p_t\}_{t=1}^T$ is determined by the policy.

\paragraph{Non-Adaptive Pricing}
If the smoothness parameter $\beta$ is known, non-adaptive dynamic pricing algorithms can achieve the optimal regret rate, which is called H\"older-Smooth Dynamic Pricing (\pcr{HSDP}) algorithm and presented in Algorithm \ref{HSDP}. The algorithm is designed based on the following idea. We first segment the price interval into many small intervals, and the length of each small bin depends on $\beta$, and then we can run local polynomial regression to approximate the true demand function in each small price interval separately.  As we formally show later in Lemma \ref{regret of HSDP}, this non-adaptive algorithm achieves the optimal regret if it is run with the correct smoothness parameter. 

\begin{algorithm}[!htb]
    \caption{H\"older-Smooth Dynamic Pricing (\pcr{HSDP})}
    \label{HSDP}
    \begin{algorithmic}[1]
    \Require{Time horizon $T$, H\"older smoothness $\beta$, minimum price $p_{\min}$, maximum demand $d_{\max}$, number of bins $N$, parameter $L>0$, optional initial history $\mathcal{D}^{(0)}$;}

    \State Set polynomial degree $k = w(\beta)$;
    \State Partition $[p_{\min},1]$ into $N$ segments of equal lengths, denoted as $\mathbf{I}_j = [a_j, b_{j}]$ where $a_j = p_{\min}+\frac{(j-1)(1-p_{\min})}{N}, b_j = p_{\min}+\frac{j(1-p_{\min})}{N}$ for $j = 1,2,\cdots, N $, and let $\Delta = L\left(\frac{1-p_{\min}}{N}\right)^{\hat{\beta}}$;
    \State Initialize segment history, realized demands and trial numbers $ \mathcal{D}_{j} := \{(p_t,d_t):p_t\in\mathbf{I}_j\}, 
    \tau_j := \sum_{p_t \in \mathcal{D}_{j}}p_td_t,
    n_j :=|\mathcal{D}_{j}|$  where $\mathcal{D}_j \subset \mathcal{D}^{(0)}$ for all $1\leq j \leq N$;

    \For {$t = 1,2,\cdots, T$} 
    \State Compute $CI_j:= [\Delta + \frac{(3d_{\max}+L)\sqrt{2}}{\sqrt{n_j}}](k+1)\ln(2(k+1)T)$;
    \State Select $j_t:=\arg\max_{1\leq j \leq N} \frac{\tau_j}{n_j} + CI_{j}$;
     \State Let $\delta = \frac{1}{T^2}$, compute $\gamma = L\sqrt{k+1} + \Delta\sqrt{|\mathcal{D}_{j_t}|} + d_{max}\sqrt{2(k+1)\ln(\frac{4(k+1)t}{\delta})} + 2$ and $\Lambda = I_{(k+1)\times (k+1)} + \sum_{(p,d)\in \mathcal{D}_{j_t}} \phi^{(k)}(p) \phi^{(k)}(p)^T$;
        \State Do local polynomial regression on $\mathbf{I}_{j_t}$ with ridge type penalty and the estimator $\hat{\theta} = \arg \min_{\theta \in \mathbb{R}^{k+1}} \sum_{(p,d)\in \mathcal{D}_{j_t}} |d - \langle \hat{\theta},\phi^{(k)}(p)\rangle|^2 + \|\theta \|_2^2$;
    \State Set price $p_t = \arg \max_{p \in \mathbf{I}_{j_t}} p \times \min \{d_{\max},\langle\hat{\theta},\phi^{(k)}(p) \rangle + \gamma\sqrt{\phi^{(k)}(p)^T\Lambda^{-1}\phi^{(k)}(p)} + \Delta \}$;
    \State Observe realized demand $d_t \in [0,d_{\max}]$;
    \State Update  $\tau_j \leftarrow \tau_j + d_tp_t, n_j \leftarrow n_j+1, \mathcal{D}_j \leftarrow \mathcal{D}_j \cup \{ (p_t,d_t)\}$ for $j = j_t$;
    \EndFor        
    \end{algorithmic}
\end{algorithm}

To help illustrate Algorithm \ref{HSDP}, we introduce the concept of local polynomial regression, a crucial component of both Algorithm \ref{HSDP} and our smoothness-adaptive dynamic pricing algorithm that will be introduced later. Compared to conventional regression methods, the local polynomial regression approach incorporates a scaling process that offers several advantages in terms of flexibility, adaptability, and efficiency. By applying the local polynomial regression with respect to a carefully chosen support set and focusing on specific small intervals, our method can effectively estimate the mean demand function with greater accuracy, making it suitable for a wide range of applications in dynamic pricing. The scaling process also allows for more efficient computation and model fitting, particularly in situations where data is limited or sparse. 
\begin{definition}[Local Polynomial Regression]
Let $\mathbb{O}  = \left\{ \left(p_{\left(1\right)},d_{\left(1\right)}\right),...\left(p_{\left(m\right)},d_{\left(m\right)}\right) \right\} $ be a sequence of observations, where $ p_{\left(1\right)}$ has support $ \subset \left[p_{\min} ,1\right]$. Our goal is to estimate $ \mathbb{E}\left[d_{\left(1\right)}|p_{\left(1\right)}\right]$ with these samples nonparametrically. Let $ \mathbf{I} = \left[a,b\right] \subset \left[p_{\min} ,1\right]$, and let those observations such that $ p_{\left(i\right)} \in \mathbf{I}$ be $ \{\mathbb{O}_{\mathbf{I}}  = \left(p_{\left(1\right)},d_{\left(1\right)}\right),...\left(p_{\left(m_0\right)},d_{\left(m_0\right)}\right)\} $. Then we can estimate  $ \mathbb{E}\left[d_{\left(1\right)}|p_{\left(1\right)}\right]$ by fitting a polynomial regression on $ \left[a,b\right]$ with samples in $ \mathbb{O}_{\mathbf{I}}$. \\
Let $t_m\left(p\right) = \left(\frac{1}{2} + \frac{p-\frac{a+b}{2}}{b-a}\right)^m$ and vector $\phi^{\left(l\right)}\left(p\right) = \left(t_0\left(p\right),t_1\left(p\right),...,t_l\left(p\right)\right)^T$ for some integer l. Define 
\begin{equation*}
    \hat{\theta} = \arg\min_{\theta \in \mathbb{R}^{l+1} } \sum_{j = 1}^{m_0}\left(d_{\left(j\right)} - \langle \phi^{\left(l\right)}\left(p_{\left(j\right)}\right),\theta \rangle\right)^2.
\end{equation*}
For concreteness, if the minimizer is not unique we define $ \hat{\theta} = 0$. The local polynomial regression estimate on $\mathbf{I}$ is given by
\begin{equation*}
    \hat{f}\left(p;\mathbb{O},l,\mathbf{I}\right):= \langle \phi^{\left(l\right)}\left(p\right), \hat{\theta} \rangle.
\end{equation*}  
\end{definition}
By using the local polynomial regression, we can leverage the H\"older smoothness condition to improve the approximation error at each small price interval. Suppose the length of a price interval is $\epsilon$, a constant approximation would lead to approximation error $O(\epsilon)$ if the demand function $f$ is Lipschitz and generally no approximation guarantee without Lipschitz condition. On the other hand, if $f\in \mathcal{H}(\beta,L)$, we can bound the approximation error by $O(\epsilon^\beta)$, which improves $O(\epsilon)$ error for $\beta> 1$, and is also strictly better when $\beta<1$ and the Lipschitz condition fails to hold. Details of this approximation guarantee are presented in Lemma \ref{polynomial approximation} in Appendix \ref{subsection:poly_approx}.

Importantly, Algorithm \ref{HSDP} relies on the input of $\beta$ to construct the number of small intervals $N$, which is decided as $\lceil T^{\frac{1}{2\beta+1}} \rceil$. The parameter $\beta$ affects the number of small intervals and therefore the length of each small interval. The underlying reason is that the approximation error of the local polynomial regression step crucially depends on the interval length, which plays an important role when establishing the optimal regret bound $\widetilde{O}(T^{\frac{\beta+1}{2\beta+1}})$. It is therefore highly nontrivial, if not impossible, to remove the dependence on $\beta$ from the design of Algorithm \ref{HSDP}.
\section{ADAPTIVITY TO UNKNOWN SMOOTHNESS}
\label{section:adaptivity}
\subsection{Difficulty of Adaption}
We show that it is impossible for any policy to achieve adaptivity without additional assumptions. This statement is formalized by establishing Theorem \ref{Impossibility of adaptation} where we consider two different smoothness levels. In Theorem \ref{Impossibility of adaptation}, we show that a policy that achieves the optimal regret rate on a smoothness level $\alpha$ could not simultaneously do so on a smoothness level $\beta<\alpha$.

\begin{theorem}
\label{Impossibility of adaptation}
It is impossible to achieve adaption without additional assumptions.
Fix any two positive H\"older smoothness parameters $\alpha > \beta >0$, and parameters $L\left(\alpha\right),L\left(\beta\right) > 0$. Suppose that there is a policy $\pi$ achieves the optimal regret $\widetilde{O}\left(T^{\frac{\alpha+1}{2\alpha+1}}\right)$ over $\mathbb{E}\left[d|p\right] = f\left(p\right)\in \mathcal{H}\left(\alpha, L\left(\alpha\right)\right)$ , then there exists a constant $C>0$ that is independent of $\pi$ such that
\[
\sup_{f\in \mathcal{H}\left(\beta,L\left(\beta\right)\right)} R^{\pi}\left(T\right)
\geq 
\Omega\left(T^{\frac{\beta+1}{2\beta+1}+\frac{\beta\left(\alpha-\beta\right)}{2\left(2\beta+1\right)^2\left(2\alpha+1\right)}}\right),
\]

which means that it cannot achieve the optimal regret over $\mathbb{E}\left[d|p\right] = f\left(p\right)\in \mathcal{H}\left(\beta, L\left(\beta\right)\right)$.
\end{theorem}

The proof of Theorem \ref{Impossibility of adaptation} is accomplished by constructing a single basis function. From this basis function, we then generate demand functions with distinct H\"older smoothness levels of $\alpha$ and $\beta$. By comparing the Kullback-Leibler divergence between the resulting probability measures under these different conditions, we establish the existence of a regret gap. The comprehensive proof of Theorem \ref{Impossibility of adaptation} can be found in Appendix~\ref{subsec:impossibility}.

The negative result in Theorem \ref{Impossibility of adaptation} highlights the difficulty of adaption and therefore necessitates the need for introducing additional conditions. A potential condition should ideally \emph{not only make the adaptivity possible }for a wide range of functions as large as possible \emph{but also not trivialize the pricing problem's complexity}.

\subsection{Self-Similarity Condition}

We identify the self-similarity condition to enable adaptivity with desirable properties. Before introducing the definition, we need some notation regarding function projections onto the space of polynomial functions. For any positive integer $l$, let $\Poly\left(l\right)$ denote the set of all polynomials of degree less than or equal to $l$. For any function $g\left(\cdot\right)$, we use $ \Gamma_l^U g\left(\cdot\right)$  to denote the $L_2$-projection of the function $g\left(\cdot\right)$ onto $\Poly\left(l\right)$ over some interval $U$, which can be computed by the following minimization
\begin{equation*}
\begin{split}
    \Gamma_l^U g\left(p\right) := & \min_q \int_U \left|g\left(u\right) - q\left(u\right) \right|^2 du,\\
   & \text{ s.t. } q\in \Poly\left(l\right).
 \end{split}
\end{equation*}

\begin{definition}[Self-Similarity Condition]
     A function $g:\left[a,b\right]  \rightarrow \mathbb{R}, \left[a,b\right] \subseteq \left[0,1\right]$ is self-similar on $\left[a,b\right]$ with parameters $\beta,l \in \mathbb{Z}^+, M_1 \in \mathbb{R}_{\geq 0}, M_2 \in \mathbb{R}_+$ if  for some positive integer $c>M_1$ it holds that 
    \begin{equation*}
        \max_{V \in \mathcal{V}_c} \sup_{p \in V} \left| \Gamma_l^V g\left(p\right) - g\left(p\right) \right| \geq M_2\cdot 2^{-c\beta},
    \end{equation*}
    where we define \[\mathcal{V}_c  = \left\{\left[a+\frac{i}{2^c},a+\frac{i+1}{2^c}\right] \cap \left[0,1\right], i = 0,1,\ldots 2^c-1\right\}\] for any positive integer $c$. We denote the class of self-similar functions by $\mathcal{S}(\beta, l ,M_1, M_2)$.
    
\end{definition}

In contrast to H\"older smoothness, the self-similarity condition provides a global lower bound on the approximation error using polynomial regression. This dual nature facilitates the estimation of the smoothness of payoff functions by comparing the approximation on different scales. The self-similarity condition has previously appeared in the literature of nonparametric regression for constructing adaptive confidence intervals \citep{picard2000adaptive, gine2010confidence}.

\begin{example}[Example of Self-Similar Functions]
Let $f$ be any function with continuous first-order derivative uniformly bounded by $C_1$. We define the function class $\mathcal{F}$  as  $\mathcal{F}( f) = \{f: x\mapsto c_0\cdot x^{\beta} + f: c_0\in \mathbb{R}, |c_0|\geq C_1 \}$, then all function in $\mathcal{F}(f)$ is self-similar with parameters $\beta, l=0$ for some constants $M_1,M_2$ depending on $C_1$ and $C_2$.
\end{example}

To enable adaptivity, we assume that the self-similar condition holds for the demand function. 
\begin{assumption}
 The demand function $f$ is self-similar with some parameters $\beta, l ,M_1, M_2$.
\end{assumption}

\subsection{Lower Bound}

We emphasize that adding this self-similarity condition does not diminish the hardness of the dynamic pricing problem. We validate this statement by showing a worst-case lower bound of \emph{any} policy $\Omega (T^{\frac{\beta+1}{2\beta+1}})$ even in the presence of the self-similarity condition. Notably, this rate matches exactly with the lower bound given in the original problem without self-similarity condition \citep{wang2021multimodal}.

\begin{theorem}[Lower Bound]
\label{Self-similarity does not lower the problem difficulty}
Self-similarity does not change the minimax regret rate and therefore does not lower the problem difficulty for any admissible dynamic pricing policy $\pi$. Formally, for any positive parameters $\beta, M_1, L>0$, there exists a constant $M_2 > 0$  satisfying that
$$
\inf_{\pi} \sup_{f\in \mathcal{H}(\beta, L) \cap \mathcal{S}(\beta, w(\beta),M_1,M_2)} R^\pi(T) \geq \Omega (T^{\frac{\beta+1}{2\beta+1}}).
$$
\end{theorem}

Theorem \ref{Self-similarity does not lower the problem difficulty} says that there exists a class of non-trivial instances belonging to the self-similar function class such that the worst-case regret of any pricing policy is lower bounded by $\Omega (T^{\frac{\beta+1}{2\beta+1}})$. A main challenge of proving Theorem \ref{Self-similarity does not lower the problem difficulty} is therefore constructing such a class of demand functions to establish the lower bound. We now explain our constructions.

Let $u\left(\cdot\right):\left[0,1\right] \rightarrow \mathbb{R}$ be a $\mathcal{C}^{\infty}$ function with $u\left(0\right)=1,u\left(1\right)=0, u^{\left(k\right)}\left(0\right) =0,u^{\left(k\right)}\left(1\right) =0,\forall k\in \mathbb{Z}^+$. Let $u_1\left(x\right) = \sin^2\left(\frac{\pi}{2}\cdot u\left(4x-1\right)\right)$.
 $\forall x \in \left[0,1\right]$, let $\sigma\left(x\right) = |x-\frac{1}{2}|$ and define 
 $$
 g\left(x\right) = 
 \begin{cases} 
c_1\cdot u_1\left(\sigma\left(x\right)\right) \cdot \frac{\left(1-\sigma\left(x\right)^\beta\right)}{2}& \text{if } \frac{1}{4}\leq \sigma\left(x\right) \leq \frac{1}{2}\\
c_1\cdot \frac{\left(1-\sigma\left(x\right)^\beta\right)}{2} & \text{if } \sigma\left(x\right) < \frac{1}{4}
\end{cases},
 $$
 for some sufficiently small constant $c_1>0$, then we have $g \in \mathcal{H}\left(\beta, L\right)$. On the other hand, we know that for some constant $M_2>0$, $g \in \mathcal{S}\left(\beta, w\left(\beta\right),0,M_2\right)$  \citep[Lemma 1.7]{gur2022smoothness}. Also, $g$ has a unique maximum point at $x=\frac{1}{2}$ and for any $x\in \{0,1\},l \in \{0,1,\cdots,w\left(\beta\right)\}$ it holds that $g^{\left(l\right)}\left(x\right) = 0$.

As a result, we have shown that for each $L, \beta>0$, there exists a function $g:\left[0,1\right] \rightarrow \left[0,1\right]$ satisfying that $g\in \mathcal{H}\left(\beta,L\right)\cap \mathcal{S}\left(\beta,w\left(\beta\right),0,M_2\right)$ for some constant $M_2>0$; and $g$ has a unique maximizer at $\frac{1}{2}$; and for any $x\in \{0,1\},k \in \{0,1,\cdots,w\left(\beta\right)\}$ it holds that $g^{\left(k\right)}\left(x\right) = 0$. The constructed function class, combined with a classical argument with the Kullback–Leibler divergence, leads to the lower bound stated in Theorem \ref{Self-similarity does not lower the problem difficulty}, and a complete proof of Theorem \ref{Self-similarity does not lower the problem difficulty} can be found in Appendix \ref{subsection:proof_lowerbound}.

\section{ALGORITHM AND REGRET ANALYSIS}
\label{section:algorithm}
In this section, we introduce our Smoothness-Adaptive Dynamic Pricing (\pcr{SADP}) algorithm and provide a detailed regret analysis. 

\subsection{Algorithm Description}

We now present our \pcr{SADP} algorithm described in Algorithm \ref{SADP}, which incorporates an efficient smoothness parameter selection phase and is designed to adapt to the unknown H\"older smoothness level. 
\begin{algorithm}[!htb]
    \caption{ Smoothness-Adaptive Dynamic Pricing (\pcr{SADP})}
    \label{SADP}
    \begin{algorithmic}[1]
    \Require{Time horizon $T$, H\"older smoothness range [$\beta_{\min},\beta_{\max}$], minimum price $p_{\min}$, maximum demand $d_{\max}$, parameter $L>0$;}

    \State Set local polynomial regression degree $l = w(\beta_{\max} )$;
    \State Set $k_1 = \frac{1}{2\beta_{\max}+2}, k_2 = \frac{1}{4\beta_{\max}+2}, K_1 = 2^{\lfloor k_1\log_2(T) \rfloor},K_2 = 2^{\lfloor k_2\log_2(T) \rfloor}$;
    \For{$i = 1,2$}
    \State Set trial time $T_i = T^{\lfloor \frac{1}{2}+k_i \rfloor}$;
    \State Pull arms $T_i$ times from $U(p_{\min},1)$ independently;
    \For{$m = 1,2,\cdots,K_i$}
    \State Let the samples which fall in $[p_{\min}+\frac{(m-1)(1-p_{\min})}{K_i},p_{\min}+\frac{m(1-p_{\min})}{K_i}]$ be $\mathbb{O}_{i,m} = \{(p_t,d_t):p_t\in [p_{\min}+\frac{(m-1)(1-p_{\min})}{K_i},p_{\min}+\frac{m(1-p_{\min})}{K_i}] \}$;
    \State Fit local polynomial regression on $[p_{\min}+\frac{(m-1)(1-p_{\min})}{K_i},p_{\min}+\frac{m(1-p_{\min})}{K_i}]$ with $\mathbb{O}_{i,m}$, construct estimate $\hat{f}_i(p)$ on the interval;
    \EndFor

    \EndFor

    \State Let $\hat{\beta} = -\frac{\ln(\max\|\hat{f}_2- \hat{f}_1\|_\infty)}{\ln(T)}- \frac{\ln(\ln(T))}{\ln(T)}$;  
    \State Set $N = \lceil T^{\frac{1}{2\hat{\beta}+1}}\rceil$, $\mathcal{D} = \cup_{i}\cup_{m}\mathbb{O}_{i,m}$;

    \State Call $\pcr{HSDP}(T-T_1-T_2,\hat{\beta},p_{\min},d_{\max},N,\Delta,L,\mathcal{D})$
               
    \end{algorithmic}
\end{algorithm}

We provide some intuition behind the design of Algorithm \ref{SADP}. Harnessing the H\"older smoothness assumption, we can employ local polynomial regression to estimate the demand function reasonably well. However, the demand function is not easily approximated by polynomials, due to the inherent self-similarity condition presented by the dual nature of H\"older smoothness. The estimation granularity refers to the number of intervals into which the domain of price is partitioned for better piecewise polynomial approximation. 

In Algorithm \ref{SADP}, we employ two distinct levels of granularity to estimate the demand function, indexed by $1$ and $2$ respectively. For estimation $i \in \{1,2\}$, the price range is segmented into small intervals of size $(1-p_{\min})/K_i$, where $K_i$ is a quantity depending on $T$. The algorithm then allocates $T_i$ time periods to collect price and demand data points and fit with local polynomial regression. The constructed estimates of demand function $f$ are $\hat{f}_1$ and $\hat{f}_2$, which helps us to establish an estimate of the H\"older smoothness parameter as defined in \eqref{eq:beta_estimate}.
\begin{equation}
\label{eq:beta_estimate}
\hat{\beta} = -\frac{\ln(\max\|\hat{f}_2- \hat{f}_1\|_\infty)}{\ln(T)}- \frac{\ln(\ln(T))}{\ln(T)}.
\end{equation}

The estimator $\hat{\beta}$ is then fed into the H\"older-smooth dynamic pricing algorithm (Algorithm \ref{HSDP}) for the remaining time horizon $T-T_1-T_2$. As evident from the algorithm design, the accuracy of $\hat{\beta}$ estimation is critical to the regret bound of our smoothness-adaptive dynamic pricing algorithm in Algorithm \ref{SADP}. In the next subsection, we provide a tight confidence interval for the estimator $\hat{\beta}$, which plays a central role in our final regret analysis of the \pcr{SADP} algorithm.

\subsection{Accuracy of Estimation}
By employing two distinct levels of granularity to estimate the demand function and, in conjunction with the previously established upper and lower bounds of the approximation error, we can prove a confidence interval for the distance between the two estimations $\|\hat{f}_2- \hat{f}_1\|_\infty$. This distance is directly related to the H\"older smoothness parameter $\beta$. Ultimately, we arrive at a reasonably narrow confidence interval for $\beta$, which converges rapidly as $T$ increases. Formally, we have the following theorem, which plays a key role in establishing the regret bound.
\begin{theorem}
\label{adaptive theorem}
    With an upper bound $\beta_{\max}$ of the smoothness parameter, under the assumptions and settings in the Algorithm \ref{SADP}, for some constant $C>0$, with probability at least $1-O\left(e^{-C \ln^2\left(T\right)}\right)$,
    \begin{equation*}
       \hat{\beta} \in \left[\beta-\frac{4\left(\beta_{\max}+1\right)\ln\left(\ln\left(T\right)\right)}{\ln\left(T\right)},\beta\right]. 
    \end{equation*}
\end{theorem}

Theorem \ref{adaptive theorem} demonstrates the effectiveness of our proposed \pcr{SADP} algorithm in estimating the H\"older smoothness parameter $\beta$ without prior knowledge. This adaptability, along with the effective smoothness parameter selection phase, enables our algorithm to construct a tight confidence interval for the H\"older smoothness parameter $\beta$ and achieve a high convergence rate. These characteristics contribute to the desired regret bound in dynamic pricing scenarios, opening up possibilities for the development of more robust and adaptive dynamic pricing algorithms.

In order to prove Theorem \ref{adaptive theorem}, we firstly introduce a lemma to characterize the convergence on $\hat{f}$.
\begin{lemma}
\label{Concentration}
Let $\{p_{\left(i\right)}, i=1,2,\cdots,n \}$ be an i.i.d. uniform sample in an interval $\mathbf{I} = \left[a,b\right] \subset \left[p_{\min} ,1\right]$, and $\mathbb{O}_{\mathbf{I}}  = \left\{\left(p_{\left(1\right)},d_{\left(1\right)}\right) ,\ldots,\left(p_{\left(n\right)},d_{\left(n\right)}\right)\right\}$. With the assumptions, suppose sub-gaussian parameter $ u_1 \leq \exp\left(u_1' \cdot n^v \right)$ for some positive constants $ \nu, u_1' $, polynomial degree $l \geq w\left(\beta\right)$.
Let
$$
 \delta_1 = \left| \mathbb{E}\left[d_{\left(1\right)} \middle|p_{\left(1\right)}=p\right] - \hat{f}\left(p;\mathbb{O}, l, \left[a,b\right]\right) \right|, 
 $$
 $$\delta_2 = \left| \Gamma_l^{\mathbf{I}} \{ \mathbb{E}\left[d_{\left(1\right)} |p_{\left(1\right)}=p\right] \} - \hat{f}\left(p;\mathbb{O}, l, \mathbf{I}\right) \right|,
$$
Then, there exist positive constants $C_1, C_2$ such that with probability at least $ 1 - O\left(e^{-C_2 \ln^2\left(n\right)}\right)$, for any $p \in \mathbf{I}$ and $n>C_1$, the following inequality holds:
\begin{equation*}
\delta_1 < \left(b-a\right)^\beta \ln\left(n\right) + \ln^3\left(n\right)\cdot n^{-\frac{1}{2}\left(1-v\right)}.
\end{equation*}
Also,
\begin{equation*}
   \delta_2 <\ln^3\left(n\right)\cdot n^{-\frac{1}{2}\left(1-v\right)}. 
\end{equation*}

\end{lemma}

With Lemma \ref{Concentration} in place, we can now proceed to prove Theorem \ref{adaptive theorem}. The proof for this theorem relies on the concentration result stated in Lemma \ref{Concentration} and the construction of the confidence interval based on the two-level granularity approach. By analyzing the relationship between the distance of the estimated demand functions and the H\"older smoothness parameter $\beta$, we can show that the estimated $\hat{\beta}$ falls within the stated confidence interval with high probability.

We also present a lemma to mitigate the issue of insufficient sample points falling within certain intervals.
\begin{lemma}
     \label{sampling}
    Let $\{B_i:i=1,2,\cdots,n\}$ be i.i.d random variables. Suppose $B_1 \sim Bernoulli(\frac{1}{m})$. Let $\Bar{B} = \frac{\sum_{i=1}^n B_i}{n}$. Then
    $$
    \mathbb{P}(\Bar{B}< \frac{1}{2m}) \leq \exp(-\frac{n}{50m}).
    $$
 \end{lemma}

\begin{proof}[Proof Sketch of Theorem \ref{adaptive theorem}]

 Our first objective is to ascertain an upper bound for the distance between $\hat{f}_1$ and $\hat{f}_2$. Define the interval $\mathbf{I}_{i,m} = \left[p_{\min}+\frac{\left(m-1\right)\left(1-p_{\min}\right)}{K_i},p_{\min}+\frac{m\left(1-p_{\min}\right)}{K_i}\right)$. Invoking the first part of Lemma \ref{Concentration} and corroborated by Lemma \ref{sampling}, with a probability of at least $1 - O\left(e^{-C\ln^2\left(n\right)}\right)$, we have $\forall p\in \mathbf{I}_{i,m}$,
\begin{equation}
\label{inequality for B}
    \begin{split}
    \left|f\left(p\right) - \hat{f}_i\left(p\right) \right|  <  K_i^{-\beta}\ln\left(T\right)  + \ln^3\left(T\right)\cdot\left(\frac{T_i}{2K_i}\right)^{-\frac{1}{2}\left(1-v_i\right)},
\end{split}
\end{equation}
for some sufficiently small constants $v_1, v_2$.

    Subsequently, utilizing inequality \eqref{inequality for B}, we deduce the following upper bound
    \begin{align*}
    \left\| \hat{f}_2 - \hat{f}_1\right\|_\infty
    & \leq \left(K_1 + K_2\right)^{-\beta}\ln\left(T\right) 
     \quad + \ln^3\left(T\right)\cdot \left[\left(\frac{T_1}{2K_1}\right)^{-\frac{1}{2}\left(1-v_1\right)} +\left(\frac{T_2}{2K_2}\right)^{-\frac{1}{2}\left(1-v_2\right)}\right] 
     \leq 2^cT^{-\beta k_2} \ln\left(T\right),
\end{align*}
    for a small constant $c$.

     However, to prove the theorem, it's necessary to establish a lower bound for the distance between and $\hat{f}_1$ and $\hat{f}_2$.

     Firstly, using inequality \eqref{inequality for B}, we can establish an upper bound for the distance between $f$ and $\hat{f}_1$. Furthermore, by invoking the second part of Lemma \ref{Concentration} as well as Lemma \ref{sampling}, with probability at least $1 - O\left(e^{-C\ln^2\left(n\right)}\right)$, $\forall p\in \mathbf{I}_{i,m}$,
        \begin{equation}
        \label{inequality for D}
            \left|\Gamma_l^{\mathbf{I}_{i,m}}f\left(p\right) - \hat{f}_2\left(p\right)\right| \leq \ln^3\left(T\right)\cdot\left(\frac{T_2}{2K_2}\right)^{-\frac{1}{2}\left(1-v_2\right)},
        \end{equation}
        where $v_1,v_2$ are sufficiently small.

    Given the self-similar properties of $f$, a lower bound for the distance between $f$ and $\Gamma_lf$ is established:
    \begin{equation}
        \label{self similarity}
        \left\| f - \Gamma_lf\right\|_\infty \geq M_2 \cdot K_2^{-\beta}.
    \end{equation}

    Subsequently, combining inequalities \eqref{inequality for B}, \eqref{inequality for D}, and \eqref{self similarity}, a lower bound can be derived as
    \begin{align*}
    \left\| \hat{f}_2 - \hat{f}_1\right\|_\infty 
     \geq M_2 \cdot K_2^{-\beta} - K_1^{-\beta}\ln\left(T\right)
     \quad - \ln^3\left(T\right)\cdot \left[\left(\frac{T_1}{2K_1}\right)^{-\frac{1}{2}\left(1-v_1\right)} \right. 
 \left.  +\left(\frac{T_2}{2K_2}\right)^{-\frac{1}{2}\left(1-v_2\right)}\right] 
     \geq \frac{M_2}{2} T^{-\beta k_2},
\end{align*}
    To advance the proof, we employ the probability union bound. For some constant $C>0$, with probability at least $1- O(e^{-C\ln^2(T)})$, the following holds:
    \begin{align*}
    \hat{\beta} 
    & =  -\frac{\ln\left(\max\left\|\hat{f}_2- \hat{f}_1\right\|_\infty\right)}{\ln\left(T\right)} 
    - \frac{\ln\left(\ln\left(T\right)\right)}{\ln\left(T\right)} \\
    & \in \left[\beta - \frac{c\ln\left(2\right) + \ln \left(\ln \left(T\right)\right)}{k_2 \ln \left(T\right)} \right. 
     \left. - \frac{\ln\left(\ln\left(T\right)\right)}{\ln\left(T\right)}, 
    \beta - \frac{\ln\left(\frac{M_2}{2}\right)}{k_2 \ln\left(T\right)} 
    - \frac{\ln\left(\ln\left(T\right)\right)}{\ln\left(T\right)}\right] \\
    & \subset 
    \left[\beta-\frac{4\left(\beta_{\max}+1\right)\ln\left(\ln\left(T\right)\right)}{\ln\left(T\right)},
    \beta\right],
\end{align*}
which completes the proof.

\end{proof}

\subsection{Regret Analysis}
After obtaining an estimation of $\beta$, we can provide a more precise bound for the distance between the local polynomial projection and the demand function. 


\begin{theorem}
\label{minimax rate of our method}
    Suppose $f \in \mathcal{H}\left(\beta, L\right)$, \pcr{SADP} has an estimation of $\beta$ with $\hat{\beta}$, and is run with $N \geq \lceil T^{\frac{1}{2\hat{\beta}+1}}\rceil,\Delta \leq L\left(\frac{1-p_{\min}}{N}\right)^{\hat{\beta}}$, then with probability $1-O\left(T^{-1}\right)$ the cumulative regret of \pcr{SADP} is upper bounded by $\Tilde{O}\left(T^{\frac{\beta+1}{2\beta+1}}\right)$.
\end{theorem}

We introduce Lemma \ref{regret of HSDP} to bound the regret rate of our non-adaptive algorithm \pcr{HSDP}, whose proof is in Appendix \ref{subsection:regretofhsdp}.
\begin{lemma}
    \label{regret of HSDP}
    If \pcr{HSDP} is run with $\hat{\beta}\leq\beta$ and other conditions stays the same as Theorem \ref{minimax rate of our method}, then with probability $1-O(T^{-1})$, the cumulative regret is upper bounded by $O(T^{\frac{\hat{\beta}+1}{2\hat{\beta}+1}})$.
\end{lemma}

\begin{proof}[Proof Sketch of Theorem \ref{minimax rate of our method}]

Considering the event $A^*: \{\hat{\beta} \in \left[\beta-\frac{4\left(\beta_{\max}+1\right)\ln\left(\ln\left(T\right)\right)}{\ln\left(T\right)},\beta\right] \}$, and by Theorem \ref{adaptive theorem}, we know that $\mathbb{P}(A^*) \geq 1 - O(e^{-C\ln^2(T)})$. Under event $A^*$, $\hat{\beta}$ converges to $\beta$ with rate $O(\frac{\ln(\ln(T))}{\ln(T)})$, we have $O(T^{\frac{\hat{\beta}+1}{2\hat{\beta}+1}} - T^{\frac{\beta+1}{2\beta+1}}) \leq O(T^{\frac{\beta+1}{2\beta+1}}\cdot T^{\frac{\beta-\hat{\beta}}{(2\hat{\beta}+1)(\beta+1)}}) \leq \tilde{O}(T^{\frac{\beta+1}{2\beta+1}})$. Considering Lemma \ref{regret of HSDP}, we can derive the regret bound for \pcr{HSDP} under event $A^*$:
$$
\sum_{t=1}^{T-T_1-T_2}  \left[p^*f\left(p^*\right) - \hat{p}_i f\left(\hat{p}_i\right)\right] \leq \tilde{O}(T^{\frac{\beta+1}{2\beta+1}}).
$$
For the adaptive part of \pcr{SADP}, note that $T_1,T_2 \leq T^{\frac{\beta+1}{2\beta+1}}$, which means that the regret is bounded by $O\left(T^{\frac{\beta+1}{2\beta+1}}\right)$. Applying the union bound with event $A^*$, we can derive that with probability $1-O\left(T^{-1}\right)$,
\begin{align*}
    R^\pi\left(T\right) &= \mathbb{E}\left[\sum_{t'=1}^{T_1+T_2}\{p^*f\left(p^*\right) - p_{t'}f\left(p_{t'}\right)\}\right] 
     + \mathbb{E}\left[\sum_{t=1}^{T-T_1-T_2}\{p^*f\left(p^*\right) - p_tf\left(p_t\right)\}\right] 
    \leq \Tilde{O}\left(T^{\frac{\beta+1}{2\beta+1}}\right).
\end{align*}
\end{proof}
Theorem \ref{minimax rate of our method} highlights the effectiveness of the \pcr{SADP} algorithm in achieving the desired regret bound under the specified conditions. By estimating the H\"older smoothness parameter $\beta$ and generalizing the non-adaptive dynamic-pricing algorithm with non-integer H\"older smoothness parameter, our algorithm is capable of maintaining a high level of performance in dynamic pricing scenarios.

\section{CONCLUSION}
\label{section:conclusion}

Motivated by the challenge of unknown smoothness levels in applications, we develop a smoothness-adaptive dynamic pricing algorithm under self-similarity conditions. To make dynamic pricing algorithms, it is very desirable to remove the parameter dependence of algorithms. Moving forward, it is promising to explore whether our approach can be generalized to other dynamic pricing problems such as feature-based dynamic pricing. To further improve adaptivity, it is of interest to consider other parameters that are implicitly used in the algorithm design.



\bibliographystyle{apalike}
\bibliography{ref}


\renewcommand{\thesection}{\Alph{section}}
\numberwithin{equation}{section}


\onecolumn
\aistatstitle{Smoothness-Adaptive Dynamic Pricing with Nonparametric Demand Learning \\
Supplementary Materials}

\setcounter{section}{0}
\renewcommand{\thesection}{\Alph{section}}

\section{PROOFS ON SELF-SIMILARITY}
\subsection{Proof of Theorem \ref{Impossibility of adaptation}}
\label{subsec:impossibility}

\begin{proof}[Proof of Theorem \ref{Impossibility of adaptation}]

 Let $d \sim \mathcal{N}\left(f\left(p\right), \sigma^2\right)$ be a normal random variable, $\sigma^2$ here is small enough such that $\mathcal{N}\left(0,\sigma^2\right)$ is sub-Gaussian under different demand settings. Also let $\epsilon_1 = \frac{\alpha-\beta}{2\left(2\alpha+1\right)\left(2\beta+1\right)}$ and $\epsilon_2 = 2\epsilon_1,\epsilon_3 = \frac{\epsilon_1}{2}$. 
    Denote
    $$
    \psi\left(p\right) = 
    \begin{cases} 
    ce^{-\frac{1}{\left(p-p_{\min}\right)\left(1-p\right)}} & \text{if } p_{\min}<p<1 \\
    0, & \text{otherwise}.
    \end{cases}
    $$
    when $c<\frac{1}{2}$ is small enough, $\psi \in \mathcal{H}\left(\alpha, L'\left(\alpha\right)\right)\cap \mathcal{H}\left(\beta,L'\left(\beta\right)\right)$. Define a counting random variable $Z_{k,m} = \sum_{t=1}^T \mathbbm{1}\left\{p_t \in \left[p_{\min}+\frac{m\left(1-p_{\min}\right)}{k},p_{\min}+\frac{\left(m+1\right)\left(1-p_{\min}\right)}{k}\right)\right\}$ for any positive integer k. Let $a \propto T^{\frac{1}{2\alpha+1}-\frac{\epsilon_2}{\alpha}}, b \propto T^{\frac{1-\epsilon_1}{2\beta+1}}$. Define index set $S_{a,b} = \{0,1,\cdots,b-1\}\cap\left(\frac{b}{a},+\infty\right)$ and let $m_0$ be the index in the index set such that $\mathbb{E}\left[Z_{k,m_0}\right]$ is the smallest. Define functions $\psi_a\left(p\right) = a^{-\alpha}\psi\left(a\left(p-p_{\min}\right)\right)$ and $\psi_b\left(p\right) = b^{-\beta}\psi\left(b\left(p-p_{\min}\right)-m_0\right)$.  And let $g_1\left(p\right) = \frac{1}{p}\left[\frac{1}{2}+\psi_a\left(p\right)\right],g_2\left(p\right) = \frac{1}{p} \left[\frac{1}{2}+\psi_a\left(p\right)+\psi_b\left(p\right)\right]$, by Lemma \ref{Holder smoothness divide by p}, $g_1 \in \mathcal{H}\left(\alpha, L\left(\alpha\right)\right), g_2 \in \mathcal{H}\left(\beta,L\left(\beta\right)\right)$.

    Denote the probability measure determined by $\mathcal{A}$ and $f = g_i$ by $\mathbb{P}_i$ for $i = 1,2$. Let $\mathbb{E}_i\left[Z\right]$ be the expectation of random variable $Z\left(p_1,\cdots,p_T,d_1,\cdots,d_T\right)$ if the probability measure is $\mathbb{P}_i$.

    If $f = g_1$, we have
    \begin{align}
    \label{R_t_a}
        R^\pi\left(T\right) 
        & = \mathbb{E}_1\left[\sum_{t = 1}^T\{p^*\cdot g_1\left(p^*\right) - p_t\cdot d_t\}\right] \nonumber \\
        & \geq \mathbb{E}\left[\sum_{t = 1}^T\{ \psi_a\left(p^*\right) -  \mathcal{N}\left(\psi_a\left(p_t\right),\sigma^2\right)\}\mathbbm{1}\left\{p_t \notin \left[p_{\min},p_{\min} + \frac{1-p_{\min}}{a}\right)\right\}\right] \nonumber\\
        & \geq \mathbb{E}\left[\sum_{t = 1}^T\psi_a\left(p^*\right)  \mathbbm{1}\left\{p_t \notin \left[p_{\min},p_{\min} + \frac{1-p_{\min}}{a}\right)\right\}\right] \nonumber\\
        & = \Omega\left(a^{-\alpha}\right)\mathbb{E}_1\left[T-Z_{a,0}\right].
    \end{align}

    By the conditions from the theorem, we have $R_T\left(\mathcal{A};g_1\right)\leq O\left(T^{\frac{\alpha+1}{2\alpha+1}+\epsilon_3}\right)$.  So we have 
    $$
    \mathbb{E}_1\left[T-Z_{a,0}\right] \leq O\left(a^\alpha\cdot T^{\frac{\alpha+1}{2\alpha+1}+\epsilon_3}\right).
    $$
    
    And by the definition of $m_0$ and notice that $b>a$, we can derive that
    \begin{equation}
    \label{bound_of_sum_z}
        \mathbb{E}_1\left[Z_{b,m_0}\right] \leq \frac{\mathbb{E}_1\left[T-Z_{a,0}\right]}{|S_{a,b}|} \leq O\left(a^\alpha\cdot b^{-1}\cdot T^{\frac{\alpha+1}{2\alpha+1}+\epsilon_3}\right).
    \end{equation}

    Then we can decompose the KL-divergence between $\mathbb{P}_1, \mathbb{P}_2$ as
    $$
    KL\left(\mathbb{P}_1||\mathbb{P}_2\right) = \mathbb{E}_1\left[\prod_{t=1}^T\frac{\mathbb{P}_1\left(d_t|p_t\right)}{\mathbb{P}_2\left(d_t|p_t\right)}\right] = \sum_{t=1}^T \mathbb{E}_1\left[KL\left(\mathcal{N}\left(g_1\left(p_t\right),\sigma^2\right)||\mathcal{N}\left(g_2\left(p_t\right),\sigma^2\right)\right)\right].
    $$

    By inequality \eqref{bound_of_sum_z} and Lemma \ref{kl_lemma_1} we can obtain
    \begin{align*}
        KL\left(\mathbb{P}_1||\mathbb{P}_2\right) 
        & =
        \frac{1}{2\sigma^2}\mathbb{E}_1\left[\sum_{t=1}^T\{g_1\left(p_t\right) - g_2\left(p_t\right)\}^2\right] \\
        & \leq \frac{b^{-2\beta}}{2\sigma^2p_{\min}} \sum_{t=1}^T \mathbb{E}_1\left[\mathbbm{1}\left\{p_t \in \left[p_{\min} + \frac{m_0\left(1-p_{\min}\right)}{b},p_{\min} +\frac{\left(m_0+1\right)\left(1-p_{\min}\right)}{b}\right)\right\}^2\right]\\
        & \leq \frac{b^{-2\beta}}{2\sigma^2p_{\min}} \mathbb{E}_1\left[Z_{b,m_0}\right]\\
        & \leq O\left(b^{-2\beta-1}\cdot a^\alpha\cdot b^{-1}\cdot T^{\frac{\alpha+1}{2\alpha+1}+\epsilon_3}\right)\\
        & = O\left(T^{\epsilon_1-1}\cdot T^{\frac{\alpha+1}{2\alpha+1}+\epsilon_3} \cdot T^{\frac{\alpha}{2\alpha+1}-\epsilon_2}\right)\\
        & = O\left(T^{\epsilon_1-\epsilon_2+\epsilon_3}\right)\\
        & = O\left(T^{-\epsilon_3}\right)\\
        & = o\left(1\right).
    \end{align*}     

    Let $A = \{Z_{b,m_0}>b^{-1}\cdot a^{\alpha}\cdot T^{\frac{\alpha+1}{2\alpha+1}+\epsilon_3} \cdot \ln\left(T\right)\}$. Inequality \ref{bound_of_sum_z} implies that $\mathbb{P}_1\left(A\right) = o\left(1\right)$. By Lemma \ref{kl_lemma_2}, we have $|\mathbb{P}_1\left(A\right) - \mathbb{P}_2\left(A\right)| = o\left(1\right)$, so we can derive that $\mathbb{P}_2\left(A\right) = o\left(1\right)$.

    Since
    $$
    b^{-1}\cdot a^{\alpha}\cdot T^{\frac{\alpha+1}{2\alpha+1}+\epsilon_3} \cdot \ln\left(T\right) = T^{\frac{\epsilon_1-1}{2\beta+1}}\cdot T^{\frac{\alpha+1}{2\alpha+1}+\epsilon_3} \cdot T^{\frac{\alpha}{2\alpha+1}-\epsilon_2} = o\left(T\right),
    $$
    On $A^c$, we have $T - Z_{b,m_0}>\frac{T}{2}$.
    Note that $\max_{p \in \left[p_{\min} ,p_{\min} +\frac{1-p_{\min}}{a}\right)}\psi_a \left(p\right) \ll \max_{p \in \left[p_{\min} +\frac{m_0\left(1-p_{\min}\right)}{b},p_{\min} +\frac{\left(m_0+1\right)\left(1-p_{\min}\right)}{b}\right)}\psi_b \left(p\right)$, then with the similar procedure in inequality \ref{R_t_a}, if $f = g_2$, we have
    \begin{align*}
        R_T\left(\mathcal{A};g_2\right) 
        & \geq \Omega\left(b^{-\beta} \cdot \mathbb{E}_2\left[T-Z_{b,m_0}\right]\right)\\
        & \geq \Omega\left(b^{-\beta} \cdot \mathbb{E}_2\left[\left(T-Z_{b,m_0}\right)\mathbbm{1}_{A^c}\right]\right) \\
        & \geq \Omega\left(b^{-\beta} \cdot \frac{T}{2} \cdot \mathbb{P}_2\left(A^c\right)\right)\\
        & \geq \Omega\left(T^{1 + \frac{\beta\left(\epsilon_1-1\right)}{2\beta+1}}\right)\\
        & \geq \Omega\left(T^{\frac{\beta+1}{2\beta+1} + \frac{\beta\left(\alpha-\beta\right)}{2\left(2\beta+1\right)^2\left(2\alpha+1\right)}}\right),
    \end{align*}
    which completes the proof.
 
\end{proof}

\subsection{Technical Lemmas for Theorem \ref{Impossibility of adaptation}}
\begin{lemma}
    \label{Holder smoothness divide by p}
    Suppose $r\left(p\right) \in \mathcal{H}\left(\beta, L'\right), p\in \left[p_{\min},1\right], 0<p_{\min}<1$, then $f\left(p\right) = \frac{r\left(p\right)}{p} \in \mathcal{H}\left(\beta,L\right)$ for some constant L.
\end{lemma}
\begin{proof}[Proof of Lemma \ref{Holder smoothness divide by p}.]
This is a basic property of H\"older class of functions, whose proof follows directly from Lemma 1.10 of \cite{gur2022smoothness} and the fact that the function $p  \mapsto 1/p$ is  H\"older smooth of any levels when restricted to the interval $[p_{\min}, 1]$.  
\end{proof}
Then we introduce three lemmas about KL-divergence which are standard results in the literature and therefore we omit the proofs.

\begin{lemma}
    \label{kl_lemma_1}
        Let $d_1 \sim \mathcal{N}\left(\mu_1,\sigma^2\right),d_2 \sim \mathcal{N}\left(\mu_2,\sigma^2\right)$, then the KL-divergence between $d_1$ and $d_2$ is $\frac{\left(\mu_1-\mu_2\right)^2}{2\sigma^2}$.
\end{lemma}

\begin{lemma}
    \label{kl_lemma_2}
        Let $\mathbb{P}_1,\mathbb{P}_2$ be two probability measures on the same $\sigma$-algebra, then for any event $A$ on this $\sigma$-algebra, we have
        $$
        KL\left(\mathbb{P}_1||\mathbb{P}_2\right) \geq 2\left(\mathbb{P}_1\left(A\right) - \mathbb{P}_2\left(A\right)\right)^2.
        $$
\end{lemma}

\begin{lemma}
    \label{Pinsker's inequality}
    For $\mathbb{P}_1,\mathbb{P}_2$ defined as above, in terms of the total variation norm $\left\|\cdot\right\|_{TV}$, we have
    $$
    \left\|\mathbb{P}_1-\mathbb{P}_2\right\|_{TV} \leq \sqrt{2KL\left(\mathbb{P}_1||\mathbb{P}_2\right)}.
    $$
\end{lemma}

\subsection{Proof of Proposition \ref{g self similar} and Lemma \ref{worst case regret 1} - \ref{worst case regret 2}}

The results in this section will be used to prove Theorem \ref{Self-similarity does not lower the problem difficulty}.
\subsubsection{Proof of Proposition \ref{g self similar}}
\begin{proposition}
\label{g self similar}
    For each $L, \beta>0$, there exists a function $g:\left[0,1\right] \rightarrow \left[0,1\right]$ satisfying that $g\in \mathcal{H}\left(\beta,L\right)\cap \mathcal{S}\left(\beta,w\left(\beta\right),0,M_2\right)$ for some constant $M_2>0$; and $g$ has a unique maximizer at $\frac{1}{2}$; and for any $x\in \{0,1\},k \in \{0,1,\cdots,w\left(\beta\right)\}$ it holds that $g^{\left(k\right)}\left(x\right) = 0$. 
\end{proposition}
\begin{proof}
 Let $u\left(\cdot\right):\left[0,1\right] \rightarrow \mathbb{R}$ be a $\mathcal{C}^{\infty}$ function with $u\left(0\right)=1,u\left(1\right)=0, u^{\left(k\right)}\left(0\right) =0,u^{\left(k\right)}\left(1\right) =0,\forall k\in \mathbb{Z}^+$. Let $u_1\left(x\right) = \sin^2\left(\frac{\pi}{2}\cdot u\left(4x-1\right)\right)$.
 $\forall x \in \left[0,1\right]$, let $\sigma\left(x\right) = |x-\frac{1}{2}|$ and define 
 $$
 g\left(x\right) = 
 \begin{cases} 
c_1\cdot u_1\left(\sigma\left(x\right)\right) \cdot \frac{\left(1-\sigma\left(x\right)^\beta\right)}{2}& \text{if } \frac{1}{4}\leq \sigma\left(x\right) \leq \frac{1}{2}\\
c_1\cdot \frac{\left(1-\sigma\left(x\right)^\beta\right)}{2} & \text{if } \sigma\left(x\right) < \frac{1}{4}
\end{cases},
 $$
 for some sufficiently small constant $c_1>0$, then we have $g \in \mathcal{H}\left(\beta, L\right)$, and following the proof procedure as Lemma 1.7 of \cite{gur2022smoothness}, for some constant $M_2>0$, $g \in \mathcal{S}\left(\beta, w\left(\beta\right),0,M_2\right)$. Also, $g$ has a unique maximum point at $x=\frac{1}{2}$ and for any $x\in \{0,1\},l \in \{0,1,\cdots,w\left(\beta\right)\}$ it holds that $g^{\left(l\right)}\left(x\right) = 0$.
\end{proof}

\subsubsection{Proof of Lemma \ref{worst case regret 1}}
\begin{lemma}
\label{worst case regret 1}
The worst-case regret of algorithm $\mathcal{A}$ over time period T can be lower bounded as
$$
\sup_{f\in \mathcal{H}\left(\beta,L\right)} R_T\left(\mathcal{A};f,\mathcal{N}\left(0,\sigma^2\right)\right) \geq \Omega\left(\epsilon_T^\beta\right) \cdot \max_{1\leq j \leq J}\left(T-\mathbb{E}_j\left[T_j\right]\right),
$$
where $f\in \mathcal{S}\left(\beta,w\left(\beta\right), 0, M_2\right)$ on $\left[p_{\min},1\right]$ for some constant $M_2$.
\end{lemma}

\begin{proof}
For $\forall j \in \{1,2,\cdots,J\}$,
\begin{align*}
\sup_{f\in \mathcal{H}\left(\beta,L\right)} R_T\left(\mathcal{A};f,\mathcal{N}\left(0,\sigma^2\right)\right) & \geq
\mathbb{E}_j\left[{\sum_{i=1}^T \left\{p^*f_j\left(p^*\right) - p_if_j\left(p_i\right)\right\}}\right]\\
& \geq \Omega\left(\epsilon_T^\beta\right) \cdot \mathbb{E}_j\left[{\sum_{i=1}^T\mathbbm{1}\left\{p_i \notin \mathbf{I}_j\right\}}\right]\\
& \geq \Omega\left(\epsilon_T^\beta\right) \cdot \left(T-\mathbb{E}_j\left[T_j\right]\right),
\end{align*} 
which completes the proof.
\end{proof}

\subsubsection{Proof of Lemma \ref{worst case regret 2}}

\begin{lemma}
\label{worst case regret 2}
For fixed $j \in \{1,2,\cdots,J\}$, we have
$$
\left|\mathbb{E}_0\left[T_j\right]-\mathbb{E}_j\left[T_j\right] \right| \leq O\left(T\epsilon_T^\beta\right) \cdot \sqrt{\mathbb{E}_0\left[T_j\right]}.
$$

\end{lemma}

\begin{proof}
Because $f_0$ and $f_j$ only differs on $\mathbf{I}_j$, we have that 
\begin{equation}
\label{kl tmp ineq 1}
    KL\left(\mathbb{P}_0||\mathbb{P}_j\right) \leq \mathbb{E}_0\left[T_j\right] \cdot \sup_{p\in \mathbf{I}_j} KL\left(\mathcal{N}\left(f_0\left(p\right),\sigma^2\right)||\mathcal{N}\left(f_j\left(p\right),\sigma^2\right)\right).
\end{equation}
Then by Lemma \ref{kl_lemma_2}, we have
\begin{equation}
\label{kl tmp ineq 2}
    \sup_{p\in \mathbf{I}_j} KL\left(\mathcal{N}\left(f_0\left(p\right),\sigma^2\right)||\mathcal{N}\left(f_j\left(p\right),\sigma^2\right)\right) \leq O\left(\epsilon_T^{2\beta}\right),
\end{equation}

and by Lemma \ref{Pinsker's inequality}, inequalities \ref{kl tmp ineq 1},\ref{kl tmp ineq 2} we have
$$
\left\|\mathbb{P}_0-\mathbb{P}_j\right\|_{TV} \leq O\left(\epsilon_T^{\beta}\right) \cdot \sqrt{\mathbb{E}_0\left[T_j\right]}.
$$
Subsequently, 
$$
\left|\mathbb{E}_0\left[T_j\right]-\mathbb{E}_j\left[T_j\right] \right| \leq \sum_{t=1}^T \{t \cdot \left|\mathbb{P}_0\left(T_j=t\right) - \mathbb{P}_j\left(T_j=t \right) \right|\} \leq T \cdot \left\|\mathbb{P}_0-\mathbb{P}_j\right\|_{TV} \leq O\left(T\epsilon_T^{\beta}\right) \cdot \sqrt{\mathbb{E}_0\left[T_j\right]},
$$
which completes the proof.
\end{proof}

\subsection{Proof of Theorem \ref{Self-similarity does not lower the problem difficulty}}

\label{subsection:proof_lowerbound}
\begin{proof}[Proof of Theorem \ref{Self-similarity does not lower the problem difficulty}]

Let $\beta$ denote the true H\"older smoothness parameter here.

Firstly, we introduce a proposition constructing the reward function we need.

Define number of intervals $J = \lceil \epsilon_T^{-1} \rceil$ for $\epsilon_T$ defined as $cT^{-\frac{1}{2\beta+1}}$ where $c$ is some sufficiently small constant depending only on $\beta$. Let $\mathbf{I}_j = \left[a_j,b_j\right],a_j = p_{\min}+\frac{\left(j-1\right)\left(1-p_{\min}\right)}{J},b_j = p_{\min} +\frac{j\left(1-p_{\min}\right)}{J}$, for $j = 1,2, \cdots J$. Then define $J$ different demand functions $f_1, f_2, \cdots, f_J$, let 
$$
f_j\left(p\right)=
\begin{cases} 
\frac{1}{2p} & \text{if } p\notin \mathbf{I}_j\\
\frac{1}{2p} + \frac{1}{p}\epsilon_T^\beta g\left(\frac{p-a_j}{\epsilon_T}\right) & \text{if } p\in \mathbf{I}_j \\
\end{cases},
$$
where $f_j \in \mathcal{H}\left(\beta,L\right) \cap \mathcal{S}\left(\beta,w\left(\beta\right),0,M_2\right)$. Define also $f_0\left(p\right) \equiv \frac{1}{2p}, p\in \left[p_{\min},1\right]$.

Denote the probability measure determined by algorithm $\mathcal{A}$ and $f = f_j$ by $\mathbb{P}_j$ for $j = 0,1,\cdots,J$. Let $\mathbb{E}_i\left[Z\right]$ be the expectation of random variable $Z\left(p_1,\cdots,p_T,d_1,\cdots,d_T\right)$ if the probability measure is $\mathbb{P}_j$. Let demand $d_i\sim \mathcal{N}\left(f\left(p_i\right),\sigma^2\right)$, where $\sigma^2$ is small enough that $\mathcal{N}\left(f\left(p_i\right),\sigma^2\right)$ is sub-Gaussian with parameters $u_1,u_2$.

Then we can upper bound the difference between $\mathbb{E}_0\left[T_j\right]$ and $\mathbb{E}_j\left[T_j\right]$ by the properties of  KL-divergence.

Let $j^* = \arg\min_{j\in\{1,2,\cdots,J\}} \mathbb{E}_0\left[T_j\right]$, it is obvious that $\mathbb{E}_0\left[T_j\right] \leq \frac{T}{J} \leq T\epsilon_T$. Then by Lemma \ref{worst case regret 2}, for some sufficiently small $c$, we can obtain
$$
\mathbb{E}_j\left[T_j\right]\leq  O\left(T\epsilon_T^\beta\right)\cdot \sqrt{\mathbb{E}_0\left[T_j\right]}\leq O\left(T\cdot c^{\beta} \cdot T^{-\frac{\beta}{2\beta+1}}\right) \cdot \sqrt{T\cdot c \cdot T^{-\frac{1}{2\beta+1}}} \leq \frac{T}{2}.
$$
Consequently, by Lemma \ref{worst case regret 1}
\begin{align*}
    \sup_{f\in \mathcal{H}\left(\beta,L\right)} R_T\left(\mathcal{A};f,\mathcal{N}\left(0,\sigma^2\right)\right) 
    & \geq \Omega\left(\epsilon_T^\beta\right) \cdot \max_{1\leq j \leq J}\left(T-\mathbb{E}_j\left[T_j\right]\right)\\
    & \geq \Omega\left(c^{\beta} \cdot T^{-\frac{\beta}{2\beta+1}}\right)\cdot \frac{T}{2}\\
    & \geq \Omega\left(T^{\frac{\beta+1}{2\beta + 1}}\right),
\end{align*}
which completes the proof.

\end{proof}

\section{PROOF OF CONSTRUCTED CONFIDENCE INTERVAL}

\subsection{Proof of Lemma \ref{Concentration}}

To prove Lemma \ref{Concentration}, we first introduce the following two lemmas.

\begin{lemma}
\label{lemma:stewart1977}
    Suppose $A,B$ are two $n_0\times n_0$ symmetric matrices. If $\left\|A-B\right\|\leq \frac{\lambda_{\min}\left(B\right)}{2}$, then
    \begin{equation*}
       \left\|A^{-1} - B^{-1}\right\| \leq 2\left(1+\sqrt{5}\right)\left\|A-B\right\| \lambda^{-2}_{\min}\left(B\right)
    \end{equation*}
    
\end{lemma}
The proof of Lemma \ref{Concentration} directly follows Theorem 3.3 of \cite{stewart1977perturbation}.
\begin{lemma}
\label{CI lemma 2}
    Suppose $d_{\left(1\right)}$ is a sub-Gaussian random variable with parameter $u_1, u_2$, then we can upper bound $\mathbb{E}\left[|d_{\left(1\right)}|\right]$ with $\frac{\left(2\sqrt{\ln\left(2u_1\right)}+1\right)}{\sqrt{u_2}}$.
\end{lemma}

\begin{proof}[Proof of Lemma \ref{CI lemma 2}]

    Let $K = 2\frac{\sqrt{\ln(2u_1)}}{u_2}$, and we have
    \begin{align*}
        \mathbb{E}[|d_{\left(1\right)}|] 
        & = \int_{0}^{\infty}\mathbb{P}(|d_{\left(1\right)}| \geq x) dx = \int_{0}^{K}\mathbb{P}(|d_{\left(1\right)}| \geq x) dx + \int_{K}^{\infty}\mathbb{P}(|d_{\left(1\right)}| \geq x) dx \\
        & \leq K + \int_{K}^{\infty} u_1 \cdot e^{-u_2x^2} dx \\
        & \leq K + \frac{u_1}{2u_2 K}e^{-u_2K^2}\\
        & \leq \frac{\left(2\sqrt{\ln\left(2u_1\right)}+1\right)}{\sqrt{u_2}}.
    \end{align*}
\end{proof}

\begin{proof}[Proof of Lemma \ref{Concentration}]
Without loss of generality, in this part, we do a translation for $d_{(1)}|p_{(1)}$ to make its expectation $0$. Let $\mathbf{P}_n$ be a $n \times \left(l+1\right)$ matrix with its $m$th row $\phi^{(l)}\left(p_{\left(m\right)}\right)^T$ for every m and $\mathbf{d}_n = \left(d_{\left(1\right)},\ldots,d_{\left(n\right)}\right)^T$. By least square regression, we obtain
\begin{equation*}
    \hat{\theta} = \left(\mathbf{P}_n^T\mathbf{P}_n\right)^{-1}\mathbf{P}_n^T\mathbf{d}_n = \left(\frac{\mathbf{P}_n^T\mathbf{P}_n}{n}\right)^{-1}\frac{\mathbf{P}^T_n\mathbf{d}_n}{n}
\end{equation*}
Define 
\begin{equation*}
    \theta_0 = \left(\mathbb{E}\left[\phi^{(l)}\left(p_{\left(1\right)}\right)\phi^{(l)}\left(p_{\left(1\right)}\right)^T\right]\right)^{-1} \mathbb{E}\left[\phi^{(l)}\left(p_{\left(1\right)}\right)^Td_{\left(1\right)}\right].
\end{equation*}
The goal of this lemma is to obtain the convergence properties of $\theta$, here we firstly prove the convergence of $\left(\frac{\mathbf{P}_n^T\mathbf{P}_n}{n}\right)^{-1}$. Let $U_1 = \phi^{(l)}\left(p_{\left(1\right)}\right)\phi^{(l)}\left(p_{\left(1\right)}\right)^T -\mathbb{E}\left[\phi^{(l)}\left(p_{\left(1\right)}\right)\phi^{(l)}\left(p_{\left(1\right)}\right)^T\right]$, by Bernstein inequality(\cite{tropp2012user}, Theorem 1.6),
\begin{equation*}
    \mathbb{P}\left(\left\|\frac{\mathbf{P}_n^T\mathbf{P}_n}{n} - \mathbb{E}\left[t\left(p_{\left(1\right)}\right) \cdot t\left(p_{\left(1\right)}\right) ^T\right] \right\| \geq w\right) \leq \left(2l+2\right) \cdot \exp \left(\frac{\frac{-nw^2}{2}}{R_1 + \frac{wR_2}{3}}\right)
\end{equation*}

where $R_1 = \max \{ \left\| \mathbb{E}\left[U_1\cdot U_1^T\right] \right\|, \left\|\mathbb{E}\left[U_1^T\cdot U_1\right] \right\| \}, R_2 = \sup \left\| U_1\right\|$. For each m, we have $t_m\left(p_{\left(1\right)}\right) \in \left[0,1\right]$ with probability 1, so $R_1, R_2 \leq O\left(1\right)$. Let $w = \frac{\ln \left(n\right)}{\sqrt{n}}$, then we have
\begin{equation*}
    \mathbb{P}\left(\left\|\frac{\mathbf{P}_n^T\mathbf{P}_n}{n} - \mathbb{E}\left[\phi^{(l)}\left(p_{\left(1\right)}\right) \cdot \phi^{(l)}\left(p_{\left(1\right)}\right) ^T\right] \right\| \geq \frac{\ln \left(n\right)}{\sqrt{n}}\right) \leq O\left(e^{-C\ln^2\left(n\right)}\right)
\end{equation*}
for some constant $C>0$.

Let $V\left(p_{\left(1\right)}\right) = \mathbb{E}\left[\phi^{(l)}\left(p_{\left(1\right)}\right)\phi^{(l)}\left(p_{\left(1\right)}\right)^T\right]$, we can prove there exists a constant $M_0$ which satisfies $\lambda_{\min}\left(V\left(p_{\left(1\right)}\right)\right) \geq M_0 > 0$ where $\lambda_{\min}$ denotes the least eigenvalue of the matrix. We have
\begin{equation*}
    \lambda_{\min}\left(V\left(p_{\left(1\right)}\right)\right) = \inf_{u \in \mathbb{R}^{l+1},\left\|u\right\|_2 = 1} u^T V\left(p_{\left(1\right)}\right)u =  \inf_{u \in \mathbb{R}^{l+1},\left\|u\right\|_2 = 1} \mathbb{E}\left\|\phi^{(l)}\left(p_{\left(1\right)}\right)^Tu\right\|^2  \geq M_0 > 0
\end{equation*}

Based on Lemma \ref{lemma:stewart1977}, we have
\begin{equation}
    \label{concentration of p_{(1)}p_{(1)}}
    \mathbb{P}\left(\left\|\left(\frac{\mathbf{P}_n^T\mathbf{P}_n}{n}\right)^{-1} - \left(\mathbb{E}\left[\phi^{(l)}\left(p_{\left(1\right)}\right) \cdot \phi^{(l)}\left(p_{\left(1\right)}\right) ^T\right]\right)^{-1} \right\| \geq C_1\frac{\ln \left(n\right)}{\sqrt{n}}\right) \leq O\left(e^{-C_2\ln^2\left(n\right)}\right)
\end{equation}

Then we prove the convergence of $\frac{\mathbf{P}^T_n\mathbf{d}_n}{n}$. Let $\mathbf{1}\left(|\mathbf{d}_n| > M\right)$ be a n-dimensional vector whose $m$th element is $\mathbf{1}\left(|d_{\left(m\right)}|> M\right)$, $F_M\left(\mathbf{d}_n\right)$ be a n-dimensional vector whose $m$th element is $d_{\left(m\right)}\mathbf{1}\left(|d_{\left(m\right)}|\leq M\right)$. Let $U_2 = \phi^{\left(l\right)}\left(p_{\left(1\right)}\right)d_{\left(1\right)}\cdot \mathbf{1}\left(|d_{\left(1\right)}|\leq M\right) - \mathbb{E}\left[\phi^{\left(l\right)}\left(p_{\left(1\right)}\right)d_{\left(1\right)}\cdot \mathbf{1}\left(|d_{\left(1\right)}|\leq M\right)\right]$, and by Bernstein Inequality(\cite{tropp2012user},Theorem 1.6),
\begin{equation*}
 \mathbb{P}\left(\left\|\frac{\mathbf{P}_n^TF_M\left(\mathbf{d}_n\right)}{n} -\mathbb{E}\left[\phi^{\left(l\right)}\left(p_{\left(1\right)}\right)d_{\left(1\right)}\cdot \mathbf{1}\left(|d_{\left(1\right)}|\leq M\right)\right] \right\| \geq w\right) \leq \left(2l+2\right) \cdot \exp \left(\frac{\frac{-nw^2}{2}}{R_3 + \frac{wR_4}{3}}\right)
\end{equation*}

where $R_3 = \max \{ \left\| \mathbb{E}\left[U_2\cdot U_2^T\right] \right\|, \left\|\mathbb{E}\left[U_2^T\cdot U_2\right] \right\| \}, R_4 = \sup \left\| U_2\right\|$. Each element of $U_2$ is upper bounded by $O\left(M\right)$ so we have $R_3 \leq O\left(M^2\right), R_4 \leq O\left(M\right)$. Let $w = \frac{M\ln \left(n\right)}{\sqrt{n}}$, then we have
\begin{equation}
\label{concentration of p_{(1)} and fmy}
    \mathbb{P}\left(\left\|\frac{\mathbf{P}_n^TF_M\left(\mathbf{d}_n\right)}{n} -\mathbb{E}\left[\phi^{\left(l\right)}\left(p_{\left(1\right)}\right)d_{\left(1\right)}\cdot \mathbf{1}\left(|d_{\left(1\right)}|\leq M\right)\right] \right\| \geq w\right) \leq O\left(e^{-C\ln^2\left(n\right)}\right)
\end{equation}
for some constant $C>0$. 
And by the sub-gaussian assumption, we have 
\begin{equation*}
    \mathbb{P}\left(\exists i \in \{1,\ldots,n\}, \left|d_{\left(i\right)}\right| > M\right) \leq u_1n\cdot e ^{-u_2M^2}
\end{equation*}
Taking $M = \sqrt{\frac{\ln\left(u_1 n\right)}{u_2}} \ln\left(n\right)$, we have
\begin{equation}
    \label{sub-gaussian of entries of d}
    \mathbb{P}\left(\exists i \in \{1,\ldots,n\}, |d_{\left(i\right)}| > M\right) \leq e^{-C\ln^2\left(n\right)}
\end{equation}
and $C$ is a constant that independent of $u_1,u_2$. And also by the we can deduce that
\begin{align*}
& \left| \mathbb{E}\left[\phi^{\left(l\right)}\left(p_{\left(1\right)}\right)d_{\left(1\right)}\cdot \mathbf{1}\left(|d_{\left(1\right)}|\leq M\right)\right] - \mathbb{E}\left[t\left(p_{\left(1\right)}\right)d_{\left(1\right)}\right] \right| \\
& = O\left(\int_M^{+\infty} p dF_d\right)  = O\left(\int_M^{+\infty} \left(p-M\right) dF_d + M\mathbb{P}\left(|d_{\left(1\right)}|\geq M\right)\right) \\
& = O\left(\int_M^{+\infty} \mathbb{P}\left(|d|\geq p\right)dp + M\mathbb{P}\left(|d_{\left(1\right)}|\geq M\right)\right) \\
& = O\left(\int_M^{+\infty} u_1\cdot e^{-u_2 p^2} dp + u_1\cdot M \cdot e^{-u_2 M^2}\right) \leq O\left(\frac{u_1}{2Mu_2}e^{-u_2M^2} + u_1M\cdot e^{-u_2M^2}\right)\\
& = O\left( \left(\frac{u_1}{2\sqrt{\ln\left(u_1 n\right)u_2} \ln\left(n\right)} + u_1\sqrt{\frac{\ln\left(u_1 n\right)}{u_2}} \ln\left(n\right)\right)\cdot \exp\left(-\ln\left(u_1n\right)\ln^2\left(n\right)\right) \right)
\end{align*}
which leads to 
\begin{equation}
    \label{dist between ind}
    \left|\mathbb{E}\left[\phi^{\left(l\right)}\left(p_{\left(1\right)}\right)d_{\left(1\right)}\cdot \mathbf{1}\left(|d_{\left(1\right)}|\leq M\right)\right] - \mathbb{E}\left[\phi^{\left(l\right)}\left(p_{\left(1\right)}\right)d_{\left(1\right)}\right] \right| \leq O\left(\frac{1}{n}\right)
\end{equation}
Combining the three inequalities \ref{concentration of p_{(1)} and fmy}, \ref{sub-gaussian of entries of d} and \ref{dist between ind}, we have the following

\begin{equation}
\label{concentration of p_{t(1)}d and Ep_{(1)}d}
    \mathbb{P}\left(\left\| \frac{\mathbf{P}_n^T\mathbf{d}_n}{n} -\mathbb{E}\left[\phi^{\left(l\right)}\left(p_{\left(1\right)}\right)d_{\left(1\right)}\right] \right\| \geq C \cdot \frac{\sqrt{\ln\left(u_1\cdot n\right)}\ln^2\left(n\right)}{\sqrt{n}}\right) \leq C_0e^{-C_2\ln^2\left(n\right)}
\end{equation}
where $C,C_0,C_2$ are constants depending on $u_2$ and $l$.

And by Lemma \ref{CI lemma 2}, we know that

\begin{equation}
\label{upper bound Et(p_{(1)})d}
    \left\| \mathbb{E}\left[\phi^{\left(l\right)}\left(p_{\left(1\right)}\right)d_{\left(1\right)}\right] \right\| \leq O\left(\frac{\sqrt{\ln\left(u_1\right)}+1}{u_2}\right).
\end{equation}

Then combining inequalities \ref{concentration of p_{(1)}p_{(1)}}, \ref{concentration of p_{t(1)}d and Ep_{(1)}d} and \ref{upper bound Et(p_{(1)})d}, recall the definition of $\theta_0$ and $\hat{\theta}$, with probability at least $1 - O\left(e^{-C_2\ln^2\left(n\right)}\right) $ for some constants $C_2$ depending on $u_2$ and $l$, we have

\begin{align}
\label{thetahat and theta0}
   & \left\| \hat{\theta} - \theta_0 \right\| = \left\| \left(\frac{\mathbf{P}_n^T\mathbf{P}_n}{n}\right)^{-1}\frac{\mathbf{P}^T_n\mathbf{d}_n}{n} - \left(\mathbb{E}\left[\phi^{\left(l\right)}\left(p_{\left(1\right)}\right)\phi^{\left(l\right)}\left(p_{\left(1\right)}\right)^T\right]\right)^{-1} \mathbb{E}\left[\phi^{\left(l\right)}\left(p_{\left(1\right)}\right)^Td_{\left(1\right)}\right]\right\| 
 \notag \\
   & = \left\| \left(\frac{\mathbf{P}_n^T\mathbf{P}_n}{n}\right)^{-1}\left(\frac{\mathbf{P}^T_n\mathbf{d}_n}{n} - \mathbb{E}\left[\phi^{\left(l\right)}\left(p_{\left(1\right)}\right)d_{\left(1\right)}\right]\right)
   + \left(\left(\frac{\mathbf{P}_n^T\mathbf{P}_n}{n}\right)^{-1} - \mathbb{E}\left[\phi^{\left(l\right)}\left(p_{\left(1\right)}\right)\phi^{\left(l\right)}\left(p_{\left(1\right)}\right)^T\right]\right)^{-1} \mathbb{E}\left[\phi^{\left(l\right)}\left(p_{\left(1\right)}\right)d_{\left(1\right)}\right]  \right\| \notag\\
   & \leq \left\| \left(\frac{\mathbf{P}_n^T\mathbf{P}_n}{n}\right)^{-1}\left(\frac{\mathbf{P}^T_n\mathbf{d}_n}{n} - \mathbb{E}\left[\phi^{\left(l\right)}\left(p_{\left(1\right)}\right)d_{\left(1\right)}\right]\right) \right\|
   + \left\| \left(\left(\frac{\mathbf{P}_n^T\mathbf{P}_n}{n}\right)^{-1} - \mathbb{E}\left[\phi^{\left(l\right)}\left(p_{\left(1\right)}\right)\phi^{\left(l\right)}\left(p_{\left(1\right)}\right)^T\right]\right)^{-1} \mathbb{E}\left[\phi^{\left(l\right)}\left(p_{\left(1\right)}\right)d_{\left(1\right)}\right] \right\| \notag \\
   & \leq O\left(\frac{\sqrt{\ln\left(u_1\cdot n\right)}\ln^2\left(n\right)}{\sqrt{n}}\right) + O\left(\frac{\left(\sqrt{\ln\left(u_1\right)}+1\right)\ln\left(n\right)}{\sqrt{nu_2}}\right) \notag \\
   &< \ln^3\left(n\right)\cdot n^{-\frac{1}{2} \left(1-v\right)},
\end{align}
for $n$ larger than some constant $C_1$ depending on $u_1', u_2, l$.

Note that $\Gamma_l^{\mathbf{I}} \{ \mathbb{E}\left[d_{\left(1\right)} |p_{\left(1\right)} = p\right] \} = \langle \phi^{\left(l\right)}\left(p\right), \theta_0 \rangle$ and $\hat{f}\left(p;\mathbb{O}, l, \mathbf{I}\right) = \langle \phi^{\left(l\right)}\left(p\right), \hat{\theta} \rangle$, with $\phi^{\left(l\right)}\left(p_{\left(1\right)}\right)\leq O\left(1\right)$, the second part of the lemma is proved.

In order to prove the first part of the lemma, we can show that $| \mathbb{E}\left[d_{\left(1\right)} |p_{\left(1\right)} = p\right]  - \langle t\left(p_{\left(1\right)}\right), \theta_0 \rangle | = O\left(\left(b-a\right)^\beta\right)$. By the Holder assumption and taylor expansion, there exists an $l+1$ dimensional vector $\theta_1$ such that $| \mathbb{E}\left[\mathbb{E}\left[d_{\left(1\right)} |p_{\left(1\right)} = p\right]\right]  - \langle \phi^{\left(l\right)}\left(p_{\left(1\right)}\right), \theta_1 \rangle | = O\left(\left(b-a\right)^\beta\right), \forall p\in \mathbf{I}$. So we have
\begin{align*}
    \left\| \theta_0 - \theta_1\right\|
    = & \left\| \left(\mathbb{E}\left[\phi^{\left(l\right)}\left(p_{\left(1\right)}\right)\phi^{\left(l\right)}\left(p_{\left(1\right)}\right)^T\right]\right)^{-1} \mathbb{E}\left[\phi^{\left(l\right)}\left(p_{\left(1\right)}\right)^Td_{\left(1\right)}\right] \right. \\
     & \left. - \left(\mathbb{E}\left[\phi^{\left(l\right)}\left(p_{\left(1\right)}\right)\phi^{\left(l\right)}\left(p_{\left(1\right)}\right)^T\right]\right)^{-1} \mathbb{E}\left[\phi^{\left(l\right)}\left(p_{\left(1\right)}\right)\phi^{\left(l\right)}\left(p_{\left(1\right)}\right)^T\right]\theta_1 \right\|\\
     = & \left\| \left(\mathbb{E}\left[\phi^{\left(l\right)}\left(p_{\left(1\right)}\right)\phi^{\left(l\right)}\left(p_{\left(1\right)}\right)^T\right]\right)^{-1} \mathbb{E}\left[\phi^{\left(l\right)}\left(p_{\left(1\right)}\right)(d_{\left(1\right)} - \langle \phi^{\left(l\right)}\left(p_{\left(1\right)}\right), \theta_1 \rangle\right] \right\| = O\left(\left(b-a\right)^\beta\right).
\end{align*}
Then by inequality \eqref{thetahat and theta0}, we can deduce that with probability at least $1 -O\left(e^{-C_2\ln^2\left(n\right)}\right)$
\begin{equation*}
    \left\| \hat{\theta} - \theta_1\right\| \leq \ln^3\left(n\right)\cdot n^{-\frac{1}{2} \left(1-v\right)} + \left(b-a\right)^\beta \ln\left(n\right).
\end{equation*}
Note that with $\phi^{\left(l\right)}\left(p_{\left(1\right)}\right)\leq O\left(1\right)$, $\left\| \langle \phi^{\left(l\right)}\left(p\right), \hat{\theta} \rangle -  \langle \phi^{\left(l\right)}\left(p\right), \theta_1 \rangle\right\| \leq O\left(\left\| \hat{\theta} - \theta_1\right\|\right)$, $\forall p \in \mathbf{I}, n>C_1$, with probability at least $1- O\left(e^{-C_2\ln^2\left(n\right)}\right)$, the following inequality holds
\begin{equation*}
    \left|  \mathbb{E}\left[d_{\left(1\right)} |p_{\left(1\right)} = p\right]  - \langle \phi^{\left(l\right)}\left(p_{\left(1\right)}\right), \hat{\theta} \rangle \right|  \leq
     \ln^3\left(n\right)\cdot n^{-\frac{1}{2} \left(1-v\right)} + \left(b-a\right)^\beta \ln\left(n\right).
\end{equation*}
Therefore the first part of the lemma is proved.

\end{proof}

\subsection{Proof of Theorem \ref{adaptive theorem}}
\begin{proof}

 We first define an event  $A = \{\exists i \in \{1,2\}, m\in 
 \{1,2,\ldots,K_i\},s.t. |\mathbb{O}_{i,m}| < \frac{T_i}{2K_i} \}$, by Lemma \ref{sampling}, we have 
    \begin{equation*}
        \mathbb{P}\left(A\right) \leq T\left(\exp\left(-\frac{T_1}{50K_1}\right) + \exp\left(-\frac{T_2}{50K_2}\right)\right),
    \end{equation*}
By conditioning on $A^c$, we can guarantee the number of samples in each interval. Next, we aim to establish an upper bound for the distance between the distance between $\hat{f}_1$ and $\hat{f}_2$. Let $\mathbf{I}_{i,m} = \left[p_{\min}+\frac{\left(m-1\right)\left(1-p_{\min}\right)}{K_i},p_{\min}+\frac{m\left(1-p_{\min}\right)}{K_i}\right)$.  Invoking the first part of Lemma \ref{Concentration}, with probability at least $1 - O\left(e^{-C\ln^2\left(n\right)}\right)$, $\forall p\in \mathbf{I}_{i,m}$, the following inequality holds:
    \begin{equation}
    \left|f\left(p\right) - \hat{f}_i\left(p\right) \right| <
    K_i^{-\beta}\ln\left(T\right) + \ln^3\left(T\right)\cdot\left(\frac{T_i}{2K_i}\right)^{-\frac{1}{2}\left(1-v_i\right)},
\end{equation}

    for some sufficiently small constants $v_1, v_2$.

    Define the event
    \begin{equation}
        B = \{  \exists i \in \{1,2\}, m\in \{1,2,\ldots K_i\},  
        p\in \mathbf{I}_{i,m}, \: \text{s.t. inequality \ref{inequality for B} does not hold} \},
\end{equation}

    Applying the union bound, we find that
    \begin{equation*}
        \mathbb{P}\left(B|A^c\right) \leq O\left(\left(K_1 + K_2\right)e^{-C\ln^2\left(n\right)}\right),
    \end{equation*}
    for some constant $C>0$.

    Then, conditioning on $A^c \cap B^c$, we can, by inequality \ref{inequality for B}, derive an upper bound as follows: 
    \begin{equation}
    \label{upper bound}
    \left\| \hat{f}_2 - \hat{f}_1\right\|_\infty <  \left(K_1 + K_2\right)^{-\beta}ln\left(T\right) 
    + \ln^3\left(T\right)\cdot\left[\left(\frac{T_1}{2K_1}\right)^{-\frac{1}{2}\left(1-v_1\right)}+\left(\frac{T_2}{2K_2}\right)^{-\frac{1}{2}\left(1-v_2\right)}\right], 
\end{equation}

     However, to prove the theorem, it's necessary to establish a lower bound for the distance between and $\hat{f}_1$ and $\hat{f}_2$. In the ensuing discussion, we aim to provide this lower bound.

     Firstly, using inequality \ref{inequality for B}, we can establish an upper bound for the distance between $f$ and $\hat{f}_1$. This, in turn, aids in deducing an upper bound for the distance between $\Gamma_lf$ and $\hat{f}_2$.

    Furthermore, by invoking the second part of Lemma \ref{Concentration} as well as Lemma \ref{sampling}, with probability at least $1 - O\left(e^{-C\ln^2\left(n\right)}\right)$, $\forall p\in \mathbf{I}_{i,m}$,
        \begin{equation}
            \left|\Gamma_l^{\mathbf{I}_{i,m}}f\left(p\right) - \hat{f}_2\left(p\right)\right| \leq \ln^3\left(T\right)\cdot\left(\frac{T_2}{2K_2}\right)^{-\frac{1}{2}\left(1-v_2\right)},
        \end{equation}
        where $v_1,v_2$ are sufficiently small.

      Define the event 
    \begin{equation*}
            D = \{ \exists p \in \left[p_{\min},1\right), \: \text{s.t. inequality \ref{inequality for D} does not hold} \},
    \end{equation*}

    And applying the union bound we have
    \begin{equation*}
        \mathbb{P}\left(D|A^c\right) \leq O\left(\left(K_1+K_2\right)e^{-C\ln^2\left(n\right)}\right),
    \end{equation*}

     Given the self-similar property of $f$, we can establish a lower bound for the distance between $f$ and $\Gamma_lf$. 

    \begin{equation}
        \left\| f - \Gamma_lf\right\|_\infty \geq M_2 \cdot K_2^{-\beta},
    \end{equation}

    Subsequently, conditioning on $A^c \cap B^c \cap D^c$, and using inequalities \ref{inequality for B}, \ref{inequality for D}, and \eqref{self similarity}, we can derive a lower bound as follows:
    \begin{equation}
    \label{lower bound}
    \left\| \hat{f}_2 - \hat{f}_1\right\|_\infty > M_2 \cdot K_2^{-\beta} -  K_1 ^{-\beta}\ln\left(T\right) 
    -\ln^3\left(T\right)\cdot\left[\left(\frac{T_1}{2K_1}\right)^{-\frac{1}{2}\left(1-v_1\right)}+\left(\frac{T_2}{2K_2}\right)^{-\frac{1}{2}\left(1-v_2\right)}\right], 
\end{equation}

    Now, let's attempt to simplify the upper bound in inequality \ref{upper bound}:
   \begin{align*}
    \left\| \hat{f}_2 - \hat{f}_1\right\|_\infty
    & \leq \left(K_1 + K_2\right)^{-\beta}\ln\left(T\right)  + \ln^3\left(T\right)\cdot \left[\left(\frac{T_1}{2K_1}\right)^{-\frac{1}{2}\left(1-v_1\right)}+\left(\frac{T_2}{2K_2}\right)^{-\frac{1}{2}\left(1-v_2\right)}\right] \\
    & \leq 2^cT^{-\beta k_2} \ln\left(T\right) ,   
\end{align*}

    for some small constant $c$.

    Similarly, we can simplify the lower bound in inequality \ref{lower bound}:
    \begin{align*}
    \left\| \hat{f}_2 - \hat{f}_1\right\|_\infty 
    & \geq M_2 \cdot K_2^{-\beta} - K_1^{-\beta}\ln\left(T\right)  - \ln^3\left(T\right)\cdot \left[\left(\frac{T_1}{2K_1}\right)^{-\frac{1}{2}\left(1-v_1\right)}+\left(\frac{T_2}{2K_2}\right)^{-\frac{1}{2}\left(1-v_2\right)}\right] \\
    & \geq \frac{M_2}{2} T^{-\beta k_2},
\end{align*}

    Thus, on the event $A^c \cap B^c \cap D^c$, we have 
    \begin{align*}
    \hat{\beta} 
    & =  -\frac{\ln\left(\max\left\|\hat{f}_2- \hat{f}_1\right\|_\infty\right)}{\ln\left(T\right)} 
    - \frac{\ln\left(\ln\left(T\right)\right)}{\ln\left(T\right)} \\
    & \in \left[\beta - \frac{c\ln\left(2\right) + \ln \left(\ln \left(T\right)\right)}{k_2 \ln \left(T\right)}  - \frac{\ln\left(\ln\left(T\right)\right)}{\ln\left(T\right)}, 
    \beta - \frac{\ln\left(\frac{M_2}{2}\right)}{k_2 \ln\left(T\right)} 
    - \frac{\ln\left(\ln\left(T\right)\right)}{\ln\left(T\right)}\right] \\
    & \subset 
    \left[\beta-\frac{4\left(\beta_{\max}+1\right)\ln\left(\ln\left(T\right)\right)}{\ln\left(T\right)},
    \beta\right],
\end{align*}

    Simultaneously we have 
    \begin{equation*}
        \mathbb{P}\left(A \cup B \cup D\right) = O \left(e^{-C\ln^2\left(n\right)}\right),
    \end{equation*}
    for some constant $C>0$ which completes the proof.

\end{proof}

\section{PROOF OF REGRET UPPER BOUNDS}

\subsection{Proof of Lemma \ref{polynomial approximation}}

\label{subsection:poly_approx}
\begin{lemma}
\label{polynomial approximation}
    Suppose $f \in \mathcal{H}\left(\beta, L\right)$ and let $\mathbf{I} = \left[a,b\right] \subset \left[p_{\min},1\right]$ be an arbitrary interval whose length is $|\mathbf{I}|$. For estimation $\hat{\beta} \leq \beta$, there exists a polynomial with degree $k = w\left(\hat{\beta}\right)$: $P_{\mathbf{I}}\left(p\right) = \Gamma_k^{\mathbf{I}}f(p) = \sum_{m = 0}^{k} a_m \left(\frac{1}{2} + \frac{p-\frac{a+b}{2}}{b-a}\right)^m$ satisfying $|a_m|m!\leq L,\forall m \leq k$, such that
    $$
    \sup_{p\in \mathbf{I}} \left|f\left(p\right) - P_{\mathbf{I}}\left(p\right) \right| \leq L\left(b-a\right)^{\hat{\beta}}.
    $$
\end{lemma}

\begin{proof}
Firstly, let $a_m = \frac{f^{\left(m\right)}\left(a\right)}{m!}\left(b-a\right)^{m}$
\begin{align*}
P_{\mathbf{I}}\left(p\right) 
& = \sum_{m = 0}^{k} a_m \left(\frac{1}{2} + \frac{p-\frac{a+b}{2}}{b-a}\right)^m \\
& = \sum_{m = 0}^{k} \frac{f^{\left(m\right)}\left(a\right)}{m!} \left(p-a\right)^m \\
\end{align*}

By Taylor expansion with Lagrangian remainders, $\forall p \in \mathbf{I}, \exists \Tilde{p} \in \mathbf{I}$ such that
$$
f\left(p\right) = \sum_{m = 0}^{k-1} \frac{f^{\left(m\right)}\left(a\right)}{m!}\left(p-a\right)^m + \frac{f^{\left(k\right)}\left(\Tilde{p}\right)}{k!}\left(p-a\right)^k
$$
With $\hat{\beta}\leq \beta$, we then have that
\begin{align*}
    |f\left(p\right) - P_{\mathbf{I}}\left(p\right) | 
    &= \frac{|f^{\left(k\right)}\left(\Tilde{p}\right) - f^{\left(k\right)}\left(a\right)|}{k!} \left(p-a\right)^{k}\\
    &= \frac{|f^{\left(w\left(\hat{\beta}\right)\right)}\left(\Tilde{p}\right) - f^{w(\hat{\beta})}\left(a\right)|}{w\left(\hat{\beta}\right)!} \left(p-a\right)^{w\left(\hat{\beta}\right)}\\
    & \leq \frac{L|p-a|^{\hat{\beta} - w\left(\hat{\beta}\right)}}{w\left(\hat{\beta}\right)!}\left(p-a\right)^{w(\hat{\beta})}\\
    & = \frac{L|b-a|^{\hat{\beta}}}{k!}\\
    & \leq L|b-a|^{\hat{\beta}}.
\end{align*}   

\end{proof}
\subsection{Proof of Lemma \ref{regret of HSDP}}
\label{subsection:regretofhsdp}
\begin{proof}
We first state two important lemmas, Lemma \ref{bound of single r} and Lemma \ref{bound of multiple r},  whose proofs will be included later in this section. 

\begin{lemma}
    \label{bound of single r}
    Suppose $f \in \mathcal{H}\left(\beta, L\right)$ and let $\mathbf{I} = \left[a,b\right] \subset \left[p_{\min},1\right]$. If \pcr{HSDP} is invoked with $\Delta \geq L\left(b-a\right)^{\hat{\beta}}$ and outputs $\hat{p}$, then with probability $1-\delta$ it holds that
    $$
    \max_{p\in \mathbf{I}} pf\left(p\right) - \hat{p}f\left(\hat{p}\right)
    \leq
    2\min\{d_{\max},\gamma\sqrt{\phi^{\left(k\right)}\left(p\right)^T\Lambda^{-1}\phi^{\left(k\right)}\left(p\right)}+\Delta\},
    $$
    where $\gamma =  L\sqrt{k+1} + \Delta\sqrt{|\mathcal{D}|} + d_{\max}\sqrt{2\left(k+1\right)\ln\left(\frac{4\left(k+1\right)t}{\delta}\right)} + 2$.
\end{lemma}

Then we can use Lemma \ref{bound of single r} to prove Lemma \ref{bound of multiple r}.

\begin{lemma}
\label{bound of multiple r}
    Keep the same setting in Lemma \ref{bound of single r} and let $\hat{p}_1,\cdots, \hat{p}_t$ be the output prices of $t$ consecutive calls on $\mathbf{I}$. Then with probability $1-O\left(T^{-1}\right)$ it holds that
    $$
    \frac{1}{t}\sum_{i=1}^t \left[\max_{p\in \mathbf{I}} pf\left(p\right) - \hat{p}_i f\left(\hat{p}_i\right)\right] \leq \left[\Delta + \frac{\left(3d_{\max}+L\right)\sqrt{2}}{\sqrt{t}}\right]\left(k+1\right)\ln\left(2\left(k+1\right)T\right),
    $$
\end{lemma}

Finally, with Lemma \ref{bound of multiple r} proved, we can do the UCB analysis for the Theorem.

For all $1\leq j \leq N$, we define $r^*\left(\mathbf{I}_j\right) = \max_{p \in \mathbf{I}_j} pf\left(p\right)$. Invoking Lemma \ref{bound of multiple r}, by concentration inequalities, with probability $1 - O\left(T^{-1}\right)$, it holds uniformly for all $j$ that 
$$
r^*\left(\mathbf{I}_j\right) \leq \frac{\tau_j}{n_j} + CI_j \leq r^*\left(\mathbf{I}_j\right) + 2CI_j,
$$
Let $T_j$ be the total number of time periods that we invoke \pcr{HSDP} in the $j$th interval, and we invoke $j_t$th interval at time $t$. Also, denote $j^* = \arg\max_j r^*\left(\mathbf{I}_j\right)$. Note that $\hat{\beta}$ is strictly less than $\beta$, then we can still derive a bound for the regret each round using UCB analysis
\begin{align*}
   r^*\left(\mathbf{I}_{j^*}\right) - r^*\left(\mathbf{I}_{j_t}\right) 
   & \leq
   \left(\frac{\tau_{j^*}}{n_{j^*}}+CI_{j^*}\right) - \left(\frac{\tau_{j_t}}{n_{j_t}}+CI_{j_t}\right) + \left(\frac{\tau_{j_t}}{n_{j_t}}+CI_{j_t}\right) - r^*\left(\mathbf{I}_{j_t}\right)\\
   & \leq \left(\frac{\tau_{j_t}}{n_{j_t}}+CI_{j_t}\right) - r^*\left(\mathbf{I}_{j_t}\right)\\
   & \leq 2CI_{j_t},
\end{align*}
And
\begin{align*}
    \sum_{t=1}^{T-T_1-T_2} CI_{j_t}
    & \leq
    \left(k+1\right)\ln\left(2\left(k+1\right)T\right)\left[\Delta T + \left(3d_{\max}+L\right) \sum_{j=1}^N \sum_{i=1}^{T_j}\sqrt\frac{2}{i}\right]\\
    & \leq \left(k+1\right)\ln\left(2\left(k+1\right)T\right)\left[\Delta T + \left(9d_{\max}+3L\right) \sum_{j=1}^N \sqrt{T_j}\right]\\
    & \leq \left(k+1\right)\ln\left(2\left(k+1\right)T\right)\left[\Delta T + \left(9d_{\max}+3L\right) \sqrt{N\sum_{j=1}^N T_j}\right]\\
    & \leq \left(k+1\right)\ln\left(2\left(k+1\right)T\right)\left(9d_{\max} + 4L \right)T^{\frac{\hat{\beta}+1}{2\hat{\beta}+1}}\\
    & \leq \Tilde{O}(T^{\frac{\hat{\beta}+1}{2\hat{\beta}+1}}).
\end{align*}

\end{proof}

\subsection{Proof of Lemma \ref{bound of single r} and \ref{bound of multiple r}}

\subsubsection{Proof of Lemma \ref{bound of single r}}

\begin{proof}
The $\left(p,d\right)$ pairs in the history are labeled as $\{\left(p_i,d_i\right)\}_{i=1}^t$ in chronological order. And we can show that $d_i = f\left(p_i\right)+\xi_i = P_{\mathbf{I}}\left(p_i\right)+\xi_i +\beta_i$, where $\{\xi_i\}_{i=1}^t$ are $i.i.d$ sub-gaussian random variables with zero mean and $|\beta_i|\leq \Delta$ with probability 1. Use vectors and matrices to denote them we have $\mathbf{p} = \left(p_i\right)_{i=1}^t, \mathbf{d} = \left(d_i\right)_{i=1}^t, \boldsymbol{\xi} = \left(\xi_i\right)_{i=1}^t , \boldsymbol{\beta} = \left(\beta_i\right)_{i=1}^t$ and $\mathbf{P} = \left(\phi^{\left(k\right)}\left(p_i\right)^T\right)_{i=1}^t \in \mathbb{R}^{t\times \left(k+1\right)}$. And the ridge estimator $\hat{\theta}$ can be written as $\hat{\theta} = \Lambda^{-1}\mathbf{P}^T\mathbf{d} = \left(\mathbf{P}^T\mathbf{P}+I\right)^{-1}\mathbf{P}^T\mathbf{d}$, plug in $\mathbf{d} = \mathbf{P}\theta^* +\boldsymbol{\xi} + \boldsymbol{\beta}$ with $\theta^*$ is the real coefficient of the expansion, we have
$$
\hat{\theta} - \theta^* = -\Lambda^{-1}\theta^* + \Lambda^{-1}\mathbf{P}^T\left(\boldsymbol{\xi} +\boldsymbol{\beta}\right),
$$
Multiplying $\left(\hat{\theta} - \theta^*\right)\Lambda$ on both sides and it leads to
\begin{align}
    \label{diff of theta 1}
    \left(\hat{\theta} - \theta^*\right)^T\Lambda\left(\hat{\theta} - \theta^*\right) 
    & = 
    - \left(\hat{\theta} - \theta^*\right)^T\theta^* + \left(\hat{\theta} - \theta^*\right)^T\mathbf{P}^T\left(\boldsymbol{\xi} +\boldsymbol{\beta}\right)\\
    &= \sum_{i=1}^t\left(\xi_i + \beta_i\right)\langle \phi^{\left(k\right)}\left(p_i\right),\hat{\theta} - \theta^* \rangle - \langle \hat{\theta} - \theta^*, \theta^* \rangle,
\end{align}
Note that
$$
\left|\sum_{i=1}^t \beta_i\langle \phi^{\left(k\right)}\left(p_i\right),\hat{\theta} - \theta^* \rangle - \langle \hat{\theta} - \theta^*, \theta^* \rangle \right|
\leq
\sqrt{\sum_{i=1}^t \beta_i^2}\sqrt{\sum_{i=1}^t|\langle \hat{\theta} - \theta^*, \theta^* \rangle|^2} \leq \Delta\sqrt{t}\cdot \sqrt{\left(\hat{\theta} - \theta^*\right)^T\Lambda\left(\hat{\theta} - \theta^*\right) },
$$

Plug the above inequality into equation (\ref{diff of theta 1}), then dividing $\sqrt{\left(\hat{\theta} - \theta^*\right)^T\Lambda\left(\hat{\theta} - \theta^*\right) }$ from both sides, 
also noting that $\Lambda \succeq I$ which makes
$\sqrt{\left(\hat{\theta} - \theta^*\right)^T\Lambda\left(\hat{\theta} - \theta^*\right)} \geq \left\|\hat{\theta} - \theta^*\right\|_2 $, we obtain

\begin{equation}
\label{inequality with phi and G}
    \sqrt{\left(\hat{\theta} - \theta^*\right)^T\Lambda\left(\hat{\theta} - \theta^*\right) }
\leq
\left\|\theta^*\right\|_2 +\Delta\sqrt{t} + \sup_{z\in \Phi_\Lambda}|G_t\left(z\right)|,
\end{equation}

where
$$
\Phi_\Lambda = \{z\in \mathbb{R}^{k+1}:z^T\Lambda z \leq 1\} , G_t\left(z\right) = \sum_{i=1}^t \xi_i \langle \phi^{\left(k\right)}\left(p_i\right), z \rangle.
$$

Recall the definition of $\theta^*$ and the H\"older class assumption, we have $\left\|\theta^*\right\|_2 \leq L\sqrt{k+1}$. In order to bound $G_t\left(z\right)$, we introduce and prove Lemma \ref{bound of G}:
\begin{lemma}
    \label{bound of G}
    Fix $k,t$ and a probability $\delta$, with probability $1-\delta$ it holds uniformly for all $\Lambda$ defined above that
    $$
    \sup_{z\in \Phi_\Lambda} \left|G_t\left(z\right) \right| \leq d_{max}\sqrt{2\left(k+1\right)\ln\left(\frac{4\left(k+1\right)t}{\delta}\right)} + 2
    $$
    
\end{lemma}
\begin{proof}[Proof of Lemma \ref{bound of G}]
By definition we have $\left\|\phi^{\left(k\right)}\left(p_i\right)\right\|_2 \leq \sqrt{k+1}$. Let $\epsilon >0$ be a small parameter. Denote $\left\|\cdot\right\|_\Lambda = \sqrt{\left(\cdot\right) \Lambda \left(\cdot\right)}$ as the $\Lambda$-norm of a vector, and $\mathbb{B}\left(r,\left\|\cdot\right\|\right) = \{z \in \mathbb{R}^{k+1}:\left\|z\right\| \leq r\}$ as a ball. Let $\mathcal{U} \subseteq \mathbb{B}\left(1,\left\|\cdot\right\|_2\right)$ be a $\epsilon$-covering of $\mathbb{B}\left(1,\left\|\cdot\right\|_2\right)$ which means that $\sup_{z \in \mathbb{B}\left(1,\left\|\cdot\right\|_2\right)} \min_{z' \in \mathcal{U}} \left\|z-z'\right\|_2 \leq \epsilon$. 
Fix arbitrary $z\in \mathcal{U}$, for $|\xi_i|\leq d_{\max}$ with probability 1, by Hoffeding's inequality we know that for any $\delta \in \left(0,1\right)$,
$$
\mathbb{P}\left(|G_t\left(z\right)|\leq d_{\max}\sqrt{2\ln\left(\frac{2}{\delta}\right)}\left\|z\right\|_\Lambda\right) \geq 1-\delta,
$$

Since $\Lambda \succeq I$, we know that $ \Phi_\Lambda = \mathbb{B}\left(1,\left\|\cdot\right\|_\Lambda\right) \subseteq \mathbb{B}\left(1,\left\|\cdot\right\|_2\right)$ and therefore $\mathcal{U}$ is also a $\epsilon$-covering of $ \Phi_\Lambda$, and it is easy to verify that there exists $\mathcal{U}$ with $\ln\left(|\mathcal{U}|\right)\leq \left(k+1\right)\ln\left(\frac{2}{\epsilon}\right)$. Applying union bound we have with probability $1-\delta$,
$$
\sup_{z\in \mathcal{U} \cap \Phi_\Lambda}|G_t\left(z\right)| \leq \sqrt{2\ln\left(\frac{2|\mathcal{U}|}{\delta}\right)} \leq  d_{max}\sqrt{2\ln\left(\frac{2|\mathcal{U}|}{\delta}\right)} \leq d_{max}\sqrt{2\left(k+1\right)\ln\left(\frac{4}{\epsilon}\right) + 2\ln\left(\frac{1}{\delta}\right)},
$$
Considering the covering, by $\left\|\phi^{\left(k\right)}\left(p_i\right)\right\|_2 \leq \sqrt{k+1}$, we have that $\left\| \Lambda \right\|_{\text{op}} \leq 1 + \left(k+1\right)t  \leq 2\left(k+1\right)t$, then we have
$$
\sup_{z\in  \Phi_\Lambda}|G_t\left(z\right)| \leq d_{max}\sqrt{2\left(k+1\right)\ln\left(\frac{4}{\epsilon}\right) + 2\ln\left(\frac{1}{\delta}\right)} + 2\left(k+1\right)t\epsilon
$$
By setting $\epsilon=\frac{1}{\left(k+1\right)t}$ we complete the proof of Lemma \ref{bound of G}.
    
\end{proof}

Then back to the proof of Lemma \ref{bound of single r}. With inequality \ref{inequality with phi and G}, invoking Lemma \ref{bound of G}, with probability $1-\delta$ we have
$$
\sqrt{\left(\hat{\theta} - \theta^*\right)^T\Lambda\left(\hat{\theta} - \theta^*\right) }
\leq L\sqrt{k+1} + \Delta\sqrt{t} + d_{max}\sqrt{2\left(k+1\right)\ln\left(\frac{4}{\epsilon}\right) + 2\ln\left(\frac{1}{\delta}\right)} + 2,
$$
And $\forall p \in \mathbf{I}$, let $\hat{f}\left(p\right) = \langle \phi^{\left(k\right)}\left(p\right), \hat{\theta} \rangle$, we can obtain
\begin{align*}
    |\hat{f}\left(p\right) - f\left(p\right)| & \leq |\hat{f}\left(p\right) - P_{\mathbf{I}}\left(p\right)| + |P_{\mathbf{I}}\left(p\right) - f\left(p\right)|\\
    & \leq \langle \phi^{\left(k\right)}\left(p\right), \hat{\theta} \rangle + \Delta \\
    & \leq \sqrt{\phi^{\left(k\right)}\left(p\right)^T\Lambda^{-1}\phi^{\left(k\right)}\left(p\right)} \sqrt{\left(\hat{\theta} - \theta^*\right)^T \Lambda \left(\hat{\theta} - \theta^*\right)} + \Delta\\
    & \leq \gamma \sqrt{\phi^{\left(k\right)}\left(p\right)^T\Lambda^{-1}\phi^{\left(k\right)}\left(p\right)} + \Delta  .
\end{align*}
The upper bound $\Bar{f}\left(p\right) = \min\{d_{\max},\langle \phi^{\left(k\right)}\left(p\right), \hat{\theta} \rangle + \gamma\sqrt{\phi^{\left(k\right)}\left(p\right)^T\Lambda^{-1}\phi^{\left(k\right)}\left(p\right)}+\Delta\}$. We can infer from the above analysis that with probability $1-\delta$, $\Bar{f}\left(p\right) \geq f\left(p\right), \forall p \in \mathbf{I}$. So $\max_{p \in \mathbf{I}} pf\left(p\right) - \hat{p}f\left(\hat{p}\right) \leq \hat{p}|\Bar{f}\left(\hat{p}\right) - f\left(p\right)|\leq 2\min\{d_{\max},\gamma\sqrt{\phi^{\left(k\right)}\left(p\right)^T\Lambda^{-1}\phi^{\left(k\right)}\left(p\right)}+\Delta\}$ which completes the proof.

\end{proof}

\subsubsection{Proof of Lemma \ref{bound of multiple r}}
\begin{proof}
    Invoke Lemma \ref{bound of single r} with $\delta = \frac{1}{T^2}$ and let $\Lambda_i = I + \sum_{i'<i} \phi^{\left(k\right)}\left(\hat{p}_{i'}\right) \phi^{\left(k\right)}\left(\hat{p}_{i'}\right)^T$ denote the $\Lambda$ matrix at the $i$th call. Denote $\gamma_{\max} = \max_{i\leq t}\gamma_i$, and we can easily verify $\gamma_{\max}\leq L\sqrt{k+1} + \Delta\sqrt{t} + d_{max}\sqrt{6\left(k+1\right)\ln\left(\left(k+1\right)T\right)}$. Recalling the right side of Lemma \ref{bound of single r}, and noting that $\gamma_{\max} \geq d_{\max}$ we have
    \begin{align*}
        \sum_{i=1}^t \min\{d_{\max},\gamma\sqrt{\phi^{\left(k\right)}\left(\hat{p}_i\right)^T\Lambda^{-1}\phi^{\left(k\right)}\left(\hat{p}_i\right)}+\Delta\}
        & \leq
        \Delta t + \sum_{i=1}^t  \min \{ {d_{\max},\gamma_{\max}\sqrt{\phi^{\left(k\right)}\left(\hat{p}_i\right)^T\Lambda^{-1}\phi^{\left(k\right)}\left(\hat{p}_i\right)} } \}\\
        & \leq \Delta t +  \gamma_{\max} \sum_{i=1}^t \min \{  {1,\sqrt{\phi^{\left(k\right)}\left(\hat{p}_i\right)^T\Lambda^{-1}\phi^{\left(k\right)}\left(\hat{p}_i\right)} } \}\\
        & \leq  \Delta t +  \gamma_{\max} \sqrt{t}  \times \sqrt{\sum_{i=1}^t \min \{  {1,\phi^{\left(k\right)}\left(\hat{p}_i\right)^T\Lambda^{-1}\phi^{\left(k\right)}\left(\hat{p}_i\right)}  \} },
    \end{align*}   
    Using the elliptical potential lemma (\cite{abbasi2012online}, Lemma 11), we know that
    $$
    \min \{1,\phi^{\left(k\right)}\left(\hat{p}_i\right)^T\Lambda^{-1}\phi^{\left(k\right)}\left(\hat{p}_i\right) \}  \leq 2\left(k+1\right)\ln\left(\left(k+1\right)t+1\right),
    $$
    Subsequently,
    \begin{align*}
        & \Delta t +  \gamma_{\max} \sqrt{t}  \times \sqrt{\sum_{i=1}^t \min \{  {1,\phi^{\left(k\right)}\left(\hat{p}_i\right)^T\Lambda^{-1}\phi^{\left(k\right)}\left(\hat{p}_i\right)}  \} }\\
        & \leq
        \Delta t + \gamma_{\max}\sqrt{t}\times \sqrt{2\left(k+1\right)\ln\left(\left(k+1\right)t+1\right)}\\
        & \leq
        \Delta t +  \left(L\sqrt{k+1} + \Delta\sqrt{t} + d_{max}\sqrt{6\left(k+1\right)\ln\left(\left(k+1\right)T\right)}\right) \times \sqrt{t} \times \sqrt{2\left(k+1\right)\ln\left(\left(k+1\right)t+1\right)} \\
        & \leq 
        \left[\left(3d_{\max}+L\right)\times \sqrt{2t} +\Delta t\right] \times
        \left(k+1\right)\ln\left(2\left(k+1\right)T\right),
    \end{align*}
    So
    $$
    \frac{1}{t}\sum_{i=1}^t \left[\max_{p\in \mathbf{I}} pf\left(p\right) - \hat{p}_i f\left(\hat{p}_i\right)\right] 
    \leq
    \left[\Delta + \frac{\left(3d_{\max}+L\right)\sqrt{2}}{\sqrt{t}}\right]\left(k+1\right)\ln\left(2\left(k+1\right)T\right).
    $$

\end{proof}


\end{document}